\documentclass[12pt]{article}
\usepackage{tikzsymbols,tikz,pgfplots}
\usetikzlibrary{snakes}
\usepackage{graphicx}
\usepackage{caption}
\usepackage{subcaption}
\pgfmathdeclarefunction{norm}{3}{%
  \pgfmathparse{sqrt(0.5*#3/pi)*exp(-0.5*#3*(#1-#2)^2)}%
}

\newcommand{\ie}{\textit{i.e.}}
\newcommand{\eg}{\textit{e.g.}}
\newcommand{\Ro}{\uppercase\expandafter{\romannumeral1}}
\newcommand{\Rt}{\uppercase\expandafter{\romannumeral2}}

\newcommand{\TMMD}{$T_{\mathrm{MMD}}$}
\newcommand{\TMMMd}{$T_{\mathrm{M}^3\mathrm{d}}$}
\newcommand{\TTMd}{$\tilde{T}_{\mathrm{M}^3\mathrm{d}}$}
\def\Pcal{\mathcal{P}}

\usepackage{bbm}
\usepackage{bm}
\usepackage[toc,page]{appendix}
\usepackage{amsthm}
\usepackage{amssymb}
\newcommand{\sgn}{\text{sgn}}
\RequirePackage[colorlinks,linkcolor=red,citecolor=red,urlcolor=red]{hyperref}
\usepackage{latexsym,epsf,dsfont,color,eurosym,mathrsfs,url}
\mathchardef\mhyphen="2D
\usepackage[round,colon,authoryear]{natbib}
\usepackage{amsmath}
\usepackage{mathtools}
\usepackage{graphicx}
\usepackage{float} 
\usepackage{caption}
\usepackage{subcaption}
\usepackage{color,soul}
\usepackage{enumerate}
\usepackage{verbatim}
\usepackage{tabularx}
\usepackage{tabu}

\newlength{\fixboxwidth}
\newtheorem{lemma}{Lemma}
\newtheorem{theorem}{Theorem}

\newcommand{\RR}{{\mathbb R}}
\newcommand{\EE}{{\mathbb E}}

\def\Bcal{\mathcal B}

\def\Fcal{\mathcal F}

\def\Hcal{\mathcal H}

\def\Pcal{\mathcal P}

\def\Xcal{\mathcal X}

\def\hat{\widehat}
\def\epsilon{\varepsilon}

\usepackage{thmtools}
\declaretheoremstyle[notefont=\bfseries,notebraces={}{},%
headpunct={},postheadspace=1em]{mystyle}
\declaretheorem[style=mystyle,numbered=no,name=Theorem]{thm-hand}

\newcounter{lastnote}

\begin{document}

\title{On the Optimality of Kernel-Embedding Based Goodness-of-Fit Tests}

\date{(\today)}

\author{Krishnakumar Balasubramanian$^\ast$, Tong Li$^\dag$ and Ming Yuan$^\ddag$\\
$^\ast$Princeton University and $^{\dag,\ddag}$Columbia University}

\footnotetext[3]{Address for Correspondence: Department of Statistics, Columbia University, 1255 Amsterdam Avenue, New York, NY 10027.}
\maketitle

\begin{abstract}
The reproducing kernel Hilbert space (RKHS) embedding of distributions offers a general and flexible framework for testing problems in arbitrary domains and has attracted considerable amount of attention in recent years. To gain insights into their operating characteristics, we study here the statistical performance of such approaches within a minimax framework. Focusing on the case of goodness-of-fit tests, our analyses show that a vanilla version of the kernel-embedding based test could be suboptimal, and suggest a simple remedy by moderating the embedding. We prove that the moderated approach provides optimal tests for a wide range of deviations from the null and can also be made adaptive over a large collection of interpolation spaces. Numerical experiments are presented to further demonstrate the merits of our approach.
\end{abstract}


\newpage

\section{Introduction}

In recent years, statistical tests based on the reproducing kernel Hilbert space (RKHS) embedding of distributions have attracted much attention because of their flexibility and broad applicability. Like other kernel methods, RKHS embedding based tests present a general and unifying framework for testing problems in arbitrary domains by using appropriate kernels defined on those domains. See, \eg, \cite{muandet2017kernel}, for a recent review and detailed discussion about the applications of kernel embeddings. The idea of using kernel embedding for comparing probability distributions was initially introduced by \cite{Smola-07,gretton2007kernel, gretton2012kernel}. Related extensions were also proposed by \cite{harchaoui2009kernel, zaremba2013b}. Furthermore,~\cite{sejdinovic2012equivalence} established a close relationship between kernel-based hypothesis tests and energy distanced based test introduced by \cite{szekely2007measuring}. See also \cite{lyons2013distance}. More recently, motivated by several applications based on quantifying the convergence of Monte Carlo simulations,~\cite{liu2016kernelized}, \cite{chwialkowski2016kernel} and \cite{gorham2017measuring} proposed goodness-of-fit tests which were based on combing the kernel based approach with Stein's identity. A linear-time method for goodness-of-fit was also proposed by \cite{jitkrittum2017linear} recently. Finally, the idea of kernel-embedding has also been used for constructing implicit generative models~\citep[\eg,][]{dziugaite2015training, li2015generative}.

Despite their popularity, fairly little is known about the statistical performance of these kernel-embedding based tests. Our goal is to fill in this void. In particular, we focus on kernel-embedding based goodness-of-fit tests and investigate their power under a general composite alternative. Our results not only provide new insights on the operating characteristics of these kernel-embedding based tests but also suggest improved testing procedures that are minimax optimal and adaptive over a large collection of alternatives.

The problem of testing for goodness-of-fit has a long and illustrious history in statistics and is often associated with household names such as \emph{Kolmogrov-Smirnov tests, Pearson's Chi-square test} or \emph{Neyman's smooth test}. A plethora of other techniques have also been proposed over the years in both parametric and non-parametric settings \citep[\eg,][]{ingster2003nptest, lehmann2008testing}. Most of the existing techniques are developed with the domain $\Xcal=\RR$ or $[0,1]$ in mind and work the best in these cases. Modern applications, however, oftentimes involve domains different from these traditional ones. For example, when dealing with directional data, which arise naturally in applications such as diffusion tensor imaging, it is natural to consider $\Xcal$ as the unit sphere in $\RR^3$ \citep[\eg,][]{jupp2005sobolev}.  Another example occurs in the context of ranking or preference data \citep[\eg,][]{Ailon2008ranking}. In these cases, $\Xcal$ can be taken as the group of permutations. Furthermore, motivated by several applications, combinatorial testing problems have been investigated recently~\citep[\eg,][]{addario2010combinatorial}, where the spaces under consideration are specific combinatorially structured spaces.

A particularly attractive approach to goodness-of-fit testing problems in general domains is through RKHS embedding of distributions. Specifically, let $K: \Xcal\times \Xcal\to \RR$ be a Mercer kernel that is symmetric, positive (semi-)definite and square integrable. The RKHS embedding of a probability measure $P$ on $(\Xcal,\Bcal)$, with respect to $K$, is given by
$$
\mu_P(\cdot) :=\int_\Xcal K(x,\cdot)P(dx).
$$
The Moore-Aronszajn Theorem indicates that there is an RKHS, denoted by $(\Hcal(K), \langle\cdot,\cdot\rangle_K)$, uniquely identified with the kernel $K$ \citep[\eg,][]{aronszajn1950theory}. It is clear that $\mu_P\in \Hcal(K)$, and hence the notion of RKHS embedding. The RKHS embedding of probability measures is closely related to a certain integral probability metric. The so-called \emph{maximum mean discrepancy} (MMD) between two probability measures $P$ and $Q$ is defined as
\begin{equation*}
\gamma_K(P,Q):=\sup_{f\in \Hcal(K): \|f\|_K\le 1} \int_\Xcal fd\left(P-Q\right),
\end{equation*}
where $\|\cdot\|_K$ is the norm associated with $(\Hcal(K), \langle\cdot,\cdot\rangle_K)$. It is not hard to see \citep[\eg,][]{gretton2012kernel} that
$$
\gamma_K(P,Q)=\|\mu_P-\mu_Q\|_K.
$$ 
The goodness-of-fit test can be carried out conveniently through RKHS embeddings of $P$ and $P_0$ by first constructing an estimate of $\gamma_K(P,P_0)$:
\begin{equation*}
\gamma_K(\hat{P}_n,P_0):=\sup_{f\in \Hcal(K): \|f\|_K\le 1} \int_\Xcal fd\left(\hat{P}_n-P_0\right),
\end{equation*}
where $\hat{P}_n$ is the empirical distribution of $X_1,\cdots,X_n$, and then rejecting $H_0$ if the estimate exceeds a threshold calibrated to ensure a certain significance level, say $\alpha$ ($0<\alpha<1$). 

In this paper, we investigate the power of the above discussed testing strategy under a general composite alternative. Following the spirit of \cite{ingster2003nptest}, we consider in particular a set of alternatives that are increasingly close to the null hypothesis. To fix ideas, we assume hereafter that $P$ is dominated by $P_0$ under the alternative so that the Radon-Nikodym derivative $dP/dP_0$ is well defined. Recall that the $\chi^2$ divergence between $P$ and $P_0$ is defined as
$$
\chi^2(P,P_0):=\int_\Xcal\left({dP\over dP_0}\right)^2dP_0-1.
$$
We are particularly interested in the detection boundary, namely how close $P$ and $P_0$ can be in terms of $\chi^2$ distance, under the alternative, so that a test based on a sample of $n$ observations can still consistently distinguish between the null hypothesis and the alternative. For example, in the parametric setting where $P$ is known up to a finite dimensional parameters under the alternative, the detection boundary of the likelihood ratio test is $n^{-1}$ under mild regularity conditions \citep[\eg,][]{lehmann2008testing}. We are concerned here with alternatives that are nonparametric in nature. Our first result suggests that the detection boundary for aforementioned $\gamma_K(\hat{P}_n,P_0)$ based test is of the order $n^{-1/2}$. However, our main results indicate, perhaps surprising at first, that this rate is far from optimal and the gap between it and the usual parametric rate can be largely bridged.

In particular, we argue that the distinguishability between $P$ and $P_0$ depends on how close $u:=dP/dP_0-1$ is to the RKHS $\Hcal(K)$. The closeness of $u$ to $\Hcal(K)$ can be measured by the the distance from $u$ to an arbitrary ball in $\Hcal(K)$. In particular, we shall consider the case where $\Hcal(K)$ is dense in $L_2(P_0)$, and focus on functions that are polynomially approximable by $\Hcal(K)$ for concreteness. More precisely, for some constants $M,\theta>0$, denote by $\Fcal(\theta; M)$ the collection of functions $f\in L_2(P_0)$ such that for any $R>0$, there exists an $f_R\in \Hcal(K)$ such that
$$
\|f_R\|_K\le R,\qquad {\rm and}\qquad \|f-f_R\|_{L_2(P_0)}\le M R^{-1/\theta}.
$$
See, \eg, \cite{cucker2007learning} for further discussion on these so-called interpolation spaces and their use in statistical learning. We shall also adopt the convention that
$$\Fcal(0;M)=\{f\in \Hcal(K): \|f\|_K\le M\}.$$
We investigate the optimal rate of detection for testing $H_0$ against
\begin{equation}
\label{alter}
H_1(\Delta_n,\theta,M): P\in \Pcal(\Delta_n,\theta,M),
\end{equation}
where $\Pcal(\Delta_n,\theta,M)$ is the collection of distributions $P$ on $(\Xcal,\Bcal)$ satisfying:
$$
dP/dP_0-1\in \Fcal(\theta; M),\qquad {\rm and}\qquad \chi^2(P,P_0)\ge \Delta_n.
$$
We call $r_n$ the optimal rate of detection if for any $c>0$, there exists no consistent test whenever $\Delta_n\le cr_n$; and on the other hand, a consistent test exists as long as $\Delta_n\gg r_n$.

Although one could consider a more general setup, for concreteness, we assume that the eigenvalues of $K$ with respect to $L_2(P_0)$ decays polynomially in that $\lambda_k\asymp k^{-2s}$. We show that the optimal rate of detection for testing $H_0$ against $H_1(\Delta_n,\theta,M)$ for any $\theta\ge 0$ is $n^{-{4 s\over 4 s+ \theta+1}}$. The rate of detection, although not achievable with a $\gamma_K(\hat{P}_n,P_0)$ based test, can be attained via a moderated version of the MMD based approach. A practical challenge to the approach, however, is its reliance on the knowledge of $\theta$. Unlike $s$ which is determined by $K$ and $P_0$ and therefore known apriori, $\theta$ depends on $u$ and is not known in advance. This naturally brings about the issue of adaptation -- is there an agnostic approach that can adaptively attain the optimal detection boundary without the knowledge of $\theta$. We show that the answer is affirmative although a small price in the form of $\log \log n$ is required to achieve such adaptation.

The rest of the paper is organized as follows. We first analyze the power of MMD based tests in Section \ref{sec:mmd}. This analysis reveals a significant gap between the detection boundary achieved by the MMD based test and the usual parametric $1/n$ rate. In turn, this prompts us to introduce, in Section \ref{sec:m3d}, a new class of tests based on a modified MMD. We show that the new tests are rate optimal. To address the practical challenge of choosing an appropriate tuning parameter for these tests, we investigate the issue of optimal adaptation in Section \ref{sec:adapt}, where we establish the optimal rates of detection for adaptively testing $H_0$ agains a broader set of alternatives and propose a test based on the modified MMD that can attain these rates. Numerical experiments are presented in Section \ref{sec:num}.
All proofs are relegated to Section \ref{sec:proof}.

\section{Operating characteristics of MMD based test}
\label{sec:mmd}

\subsection{Background and notation}
In this section, we investigate the performance of the MMD based test. As shown in \citet{gretton2012kernel}, the squared MMD between two probability distributions $P$ and $P_0$ can be expressed as
\begin{align}
\gamma_K^2(P,P_0)=\int K(x,x'){d}(P-P_0)(x){d}(P-P_0)(x').\label{deg}
\end{align}
Write
\begin{align*}
\bar{K}(x,x')=K(x,x')-\EE _{P_0}K(x,X)-\EE _{P_0} K(X,x')+\EE _{P_0}K(X,X'),
\end{align*}
where the subscript $P_0$ signifies the fact that the expectation is taken over $X, X'\sim P_0$ independently. By (\ref{deg}), $\gamma_K^2(P,P_0)=\gamma_{\bar{K}}^2(P,P_0)$. Therefore, without loss of generality, we shall assume in what follows that $K$ is degenerate under $P_0$, \ie, 
\begin{align}
\EE _{P_0}K(X,\cdot)=0.\label{degenerate}
\end{align}
For brevity, we shall omit the subscript $K$ in $\gamma$ in the rest of the paper, unless it is necessary to emphasize the dependence of MMD on the reproducing kernel.

Assuming that $K$ is square integrable, by Mercer's theorem, it can be decomposed as
\begin{align}
K(x,x')=\sum\limits_{k\geq 1}\lambda_k\varphi_k(x)\varphi_k(x'),\label{decomp}
\end{align}
where the limit is in the sense of $L_2(P_0)$, $\lambda_1>\lambda_2>\cdots>0$ are the positive eigenvalues of the integral operator induced by $K$, and $\{\varphi_k: k\ge 1\}$ are the corresponding orthonormal eigenfunctions, \ie, $\langle \varphi_k,\varphi_{k'} \rangle_{L_2(P_0)}=\delta_{k,k'}$ and $\langle \varphi_k,\varphi_{k'} \rangle_K=\lambda_k^{-1}\delta_{k,k'}$, with $\delta$ representing the Kronecker delta. For the sake of concreteness, we shall assume $K$ is universal in that $\{\varphi_k: k\ge 1\}$ forms an orthonormal basis of $L_2(P_0)$, and has infinitely many positive eigenvalues decaying polynomially, that is,
\begin{align}
0<\varliminf_{k\rightarrow \infty}k^{2s}\lambda_k\leq\varlimsup_{k\rightarrow \infty}k^{2s}\lambda_k<\infty\label{summable}
\end{align}
for some $s>1/2$. Moreover, we assume the eigenfunctions are uniformly bounded, \ie,
\begin{align}
\sup_{k\geq 1}\|\varphi_k\|_{\infty}<\infty.\label{unif}
\end{align}
Assumptions (\ref{summable}) and (\ref{unif}) ensure that the spectral decomposition (\ref{decomp}) holds both pointwisely and uniformly.

Note that (\ref{degenerate}) implies $\EE _{P_0}\varphi_k(X)=0$, $\forall\ k\ge 1$, and (\ref{decomp}) gives
\begin{align*}
\gamma^2(P,P_0)=\sum\limits_{k\geq 1}\lambda_k[\EE _{P}\varphi_k(X)]^2
\end{align*}
for any $P$. Accordingly, when $P$ is replaced by the empirical distribution $\hat{P}_n$, the empirical squared MMD can be expressed as
\begin{align*}
\gamma^2(\hat{P}_n,P_0)=\sum\limits_{k\geq 1}\lambda_k\left[\frac{1}{n}\sum\limits_{i=1}^{n}\varphi_k(X_i)\right]^2.
\end{align*}
Classic results on the asymptotics of V-statistic \citep{serfling2009approximation} imply that
\begin{align*}
n\gamma^2(\hat{P}_n,P_0)\stackrel{d}{\rightarrow}\sum\limits_{k\geq 1}\lambda_kZ_k^2:= W
\end{align*}
under $H_0$, where $Z_k\stackrel{i.i.d.}{\sim} N(0,1)$. Let $T_{\mathrm{MMD}}$ be an MMD based test, which rejects $H_0$ if and only if $n\gamma^2(\hat{P}_n,P_0)$ exceeds the upper $\alpha$ quantile $q_{w,1-\alpha}$ of $W$, \ie,
$$T_{\mathrm{MMD}}=\mathds{1}_{\{n\gamma^2(\hat{P}_n,P_0)>q_{w,1-\alpha}\}}.$$
The above limiting distribution of $n\gamma^2(\hat{P}_n,P_0)$ immediately suggests that $T_{\mathrm{MMD}}$ is an asymptotic $\alpha$-level test.

\subsection{Power analysis for MMD based tests}
We now investigate the power of $T_{\mathrm{MMD}}$ in testing $H_0$ against $H_1(\Delta_n,\theta,M)$ given by (\ref{alter}). Recall that the type \Rt\  error of a test $T:\mathcal{X}^{n}\rightarrow [0,1]$ for testing $H_0$ against a composite alternative $H_1:P\in\mathcal{P}$ is given by
$$\beta(T;\Pcal)=\sup\limits_{P\in \mathcal{P}}\EE_{P}[1-T(X_1,\ldots,X_n)],$$
where $\EE_P$ means taking expectation over $X_1,\ldots, X_n\stackrel{i.i.d.}{\sim}P$. 
For brevity, we shall write $\beta(T;\Delta_n,\theta,M)$ instead of $\beta(T;\mathcal{P}(\Delta_n,\theta,M))$ in what follows. The performance of a test $T$ can then be evaluated by its detection boundary, that is, the smallest $\Delta_n$ under which the type \Rt\ error converges to $0$ as $n\rightarrow \infty$. Our first result establishes the convergence rate of the detection boundary for \TMMD\ in the case when $\theta=0$. Hereafter, we abbreviate $M$ in $\Pcal(\Delta_n,\theta,M)$, $H_1(\Delta_n,\theta,M)$ and $\beta(T;\Delta_n,\theta,M)$, unless it is necessary to emphasize the dependence.


\begin{theorem}\label{mmdthm}
Consider testing $H_0: P=P_0$ against $H_1(\Delta_n,0)$ by \TMMD.
\renewcommand{\labelenumi}{(\roman{enumi})} 
\begin{enumerate}
	\item If $\sqrt{n}\Delta_n\rightarrow \infty$, then
	$$\beta(T_{\mathrm{MMD}};\Delta_n,0)\rightarrow 0\qquad  {\rm as} \quad n\rightarrow \infty;$$
	\item conversely, there exists a constant $c_0>0$ such that
	$$\varliminf\limits_{n\rightarrow \infty}\beta(T_{\mathrm{MMD}};c_0n^{-1/2},0)>0.$$
\end{enumerate}
\end{theorem}

Theorem \ref{mmdthm} shows that when the alternative $H_1(\Delta_n,0)$ is considered, the detection boundary of \TMMD\ is of the order $n^{-1/2}$.
%
%
%
It is of interest to compare the detection rate achieved by $T_{\text{MMD}}$ with that in a parametric setting where  consistent tests are available if $n\Delta_n\rightarrow \infty$ \citep{lehmann2008testing}. It is natural to raise the question to what extent such a gap can be entirely attributed to the fundamental difference between parametric and nonparametric testing problems. We shall now argue that this gap actually is largely due to the sup-optimality of \TMMD, and the detection boundary of  \TMMD\ could be significantly improved through a slight modification of the MMD.

\section{Optimal tests based on moderated MMD}
\label{sec:m3d}

\subsection{Moderated MMD test statistic}
The basic idea behind MMD is to project two probability measures onto a unit ball in $\Hcal(K)$ and use the distance between the two projections to measure the distance between the original probability measures. If the two probability measures are far away from $\Hcal(K)$, the distance between the two projections may not honestly reflect the distance between them. More specifically, $\gamma^2(P,P_0)=\sum\limits_{k\geq 1}\lambda_k[\EE _{P}\varphi_k(X)]^2$, while the $\chi^2$ distance between $P$ and $P_0$ is $\chi^2(P,P_0)=\sum\limits_{k\geq 1}[\EE _{P}\varphi_k(X)]^2$. Considering that $\lambda_k$ decreases with $k$, $\gamma^2(P,P_0)$ can be much smaller than $\chi^2(P,P_0)$. To overcome this problem, we consider a moderated version of the MMD which allows us to project the probability measures onto a larger ball in $\Hcal(K)$. The new class of integral probability metric between two distributions $P$ and $Q$ is given as
\begin{equation}
\label{def2}
\eta_{K,\varrho}(P,Q;P_0)=\sup_{f\in \Hcal(K): \|f\|_{L_2(P_0)}^2+\varrho^2\|f\|_K^2\le 1} \int_\Xcal fd(P-Q)
\end{equation}
for a given distribution $P_0$ and a constant $\varrho>0$. It should be noted that a related test statistics was proposed previously by \cite{harchaoui2009kernel} from a completely different viewpoint.

It is worth noting that $\eta_{K,\varrho}(P,Q;P_0)$ can also be identified with a particular type of RKHS embedding. Specifically, $\eta_{K,\varrho}(P,Q;P_0)=\gamma_{\tilde{K}_{\varrho}}(P,Q)$, where
\begin{align*}
\tilde{K}_{\varrho}(x,x'):=\sum\limits_{k\geq 1}\frac{\lambda_k}{\lambda_k+\varrho^2}\varphi_k(x)\varphi_k(x').
\end{align*}
We shall abbreviate the dependence of $\eta$ on $K$ and $P_0$ unless necessary. The unit ball in (\ref{def2}) is defined in terms of both RKHS norm and $L^2$ norm. Recall that $u=dP/dP_0-1$ so that
\begin{align*}
\sup\limits_{\|f\|_{L_2(P_0)}\leq 1}\int\limits_{\mathcal{X}}f{d}(P-P_0)=\sup\limits_{\|f\|_{L_2(P_0)}\leq 1}\int\limits_{\mathcal{X}}fu{d}P_0=\|u\|_{L_2(P_0)}=\chi(P,P_0).
\end{align*}
We can therefore expect that a smaller $\varrho$ will make $\eta^2_{\varrho}(P,P_0)$ closer to $\chi^2(P,P_0)$, since the unit ball to be considered will become more similar to the unit ball in $L_2(P_0)$. This can also be verified by noticing that $\eta^2_{\varrho}(P,P_0)$ could be expressed as
\begin{align*}
\eta^2_{\varrho}(P,P_0)=\sum\limits_{k\geq 1}\frac{\lambda_k}{\lambda_k+\varrho^2}[\EE_P\varphi_k(X)]^2.
\end{align*}
Therefore, we choose $\varrho$ converging to $0$ when constructing our test statistic.

Hereafter we shall attach the subscript $n$ to $\varrho$ to signify its dependence on $n$. We now argue that letting $\varrho_n$ converge to $0$ at an appropriate rate indeed results in a test more powerful than \TMMD. The test statistic we propose is the empirical version of $\eta^2_{\varrho_n}(P,P_0)$,
\begin{align}
\eta^2_{\varrho_n}(\hat{P}_n,P_0)=\sum\limits_{k\geq 1}\frac{\lambda_k}{\lambda_k+\varrho_n^2}\left[\frac{1}{n}\sum\limits_{i=1}^{n}\varphi_k(X_i)\right]^2.\label{eq:etatest}
\end{align}

\subsection{Operating characteristics of $\eta^2_{\varrho_n}(\hat{P}_n,P_0)$ based tests}
Although the expression for $\eta^2_{\varrho_n}(\hat{P}_n,P_0)$ given by \eqref{eq:etatest} looks similar to that of $\gamma^2(\hat{P}_n,P_0)$, their asymptotic behaviors are quite different. At a technical level, this is due to the fact that the eigenvalues of the underlying kernel
$$\tilde{\lambda}_{nk}:=\frac{\lambda_k}{\lambda_k+\varrho_n^2}$$ 
depend on $n$ and may not be uniformly summable over $n$. As presented in the following theorem, a certain type of asymptotic normality, instead of a sum of chi-squares as in the case of $\gamma^2(\hat{P}_n, P_0)$, holds for $\eta^2_{\varrho_n}(\hat{P}_n,P_0)$ under $P_0$, which helps determine the rejection region of the $\eta^2_{\varrho_n}$ based test.

%
%
\begin{theorem}\label{asympm3d}
Assume that $\varrho_n\to 0$ as $n\to \infty$ in such a fashion that $n\varrho_n^{{1}/(2s)}\rightarrow\infty$. Then under $H_0$ where $X_1,\ldots, X_n\stackrel{i.i.d.}{\sim}P_0$, 
\begin{align*}
v_n^{-1/2}[n\eta_{\varrho_n}^2(\hat{P}_n,P_0)-A_n]\stackrel{d}{\rightarrow}N(0,2),
\end{align*}
where
$$v_n=\sum\limits_{k\geq 1}\left(\frac{\lambda_k}{\lambda_k+\varrho_n^2}\right)^2,\qquad {\rm and} \qquad
A_n=\frac{1}{n}\sum\limits_{i=1}^n\tilde{K}_{\varrho_n}(X_i,X_i).$$
%
\end{theorem}

In the light of Theorem \ref{asympm3d}, a test that rejects $H_0$ if and only if
$$2^{-1/2}v_n^{-1/2}[n{\eta}_{\varrho_n}^2(\hat{P}_n,P_0)-A_n]$$
exceeds $z_{1-\alpha}$ is an asymptotic $\alpha$-level test, where $z_{1-\alpha}$ stands for the $1-\alpha$ quantile of a standard normal distribution. We refer to this test as \TMMMd. The performance of \TMMMd\ under the alternative hypothesis is characterized by the following theorem, showing that its detection boundary is much improved when compared with that of \TMMD.

\begin{theorem}\label{crm3d}
Consider testing $H_0$ against $H_1(\Delta_n,\theta)$ by \TMMMd\ with $\varrho_n=cn^{-{2s(\theta+1)\over 4 s+\theta+1}}$ for an arbitrary constant $c>0$. If $n^{\frac{4s}{4s+\theta+1}}\Delta_n\rightarrow\infty$, then
$$
\beta(T_{\mathrm{M}^3\mathrm{d}};\Delta_n, \theta)\to 0, \qquad {\rm as\ }n\to \infty.
$$
\end{theorem}

Theorem \ref{crm3d} indicates that the detection boundary for \TMMMd\ is $n^{-{4s}/({4s+\theta+1})}$. In particular, when testing $H_0$ against $H_1(\Delta_n, 0)$, \ie, $\theta=0$, it becomes $n^{-4s/(4s+1)}$. This is to be contrasted with the detection boundary for \TMMD, which, as suggested by Theorem \ref{mmdthm}, is of the order $n^{-1/2}$. It is also worth noting that the detection boundary for \TMMMd\ deteriorates as $\theta$ increases, implying that it is harder to test against a larger interpolation space.

\subsection{Minimax optimality}

It is of interest to investigate if the detection boundary of \TMMMd\ can be further improved. We now show that the answer is negative in a certain sense. More specifically, we shall follow the minimax framework for nonparametric hypothesis testing pioneered by Ingster \citep[see, \eg,][]{ingster1993asymptotically, ingster1995minimax} and show that \TMMMd\ attains the optimal rate of detection for testing $H_0$ against $H_1(\Delta_n,\theta)$ in that no consistent test exists if there exists $c>0$ such that $\Delta_n\le cn^{-\frac{4s}{4s+\theta+1}}$.

\begin{theorem}\label{cr}
Consider testing $H_0: P=P_0$ against $H_1(\Delta_n,\theta)$, for some $\theta<2s-1$. If $\varlimsup\limits_{n\rightarrow \infty}\Delta_nn^{4s\over 4s+\theta+1}<\infty$, then
$$
\varliminf\limits_{n\rightarrow \infty}\inf_{{T}\in\mathcal{T}_n} \left[\EE_{P_0} {T} +\beta({T};\Delta_n,\theta)\right]>0,
$$
where $\mathcal{T}_n$ denotes the collection of all test functions based on $X_1,\ldots,X_n$.
%
\end{theorem}

Recall that for a test ${T}$, $\EE_{P_0} {T}$ is its Type I error. Theorem \ref{cr} shows that, if $\Delta_n=O\left(n^{-4s/(4s+\theta+1)}\right)$, then the sum of Type I and Type II errors of any test does not vanish as $n$ increases. In other words, there is no consistent test if $\Delta_n=O\left(n^{-4s/(4s+\theta+1)}\right)$. Together with Theorem \ref{crm3d}, this suggests that \TMMMd\ is rate optimal in the minimax sense.

\section{Adaptation}
\label{sec:adapt}
Despite the minimax optimality of \TMMMd, a practical challenge in using it is the choice of an appropriate tuning parameter $\varrho_n$. In particular, Theorem \ref{crm3d} suggests that $\varrho_n$ needs to be taken at the order of $n^{-2s(\theta+1)/(4s+\theta +1)}$ which depends on the value of $s$ and $\theta$. On the one hand, since $P_0$ and $K$ are known apriori, so is $s$. On the other hand, $\theta$ reflects the property of $dP/dP_0$ which is typically not known in advance. This naturally brings about the issue of adaptation \citep[see, \eg,][]{spokoiny1996adaptive,ingster2000adaptive}. In other words, we are interested in a single testing procedure that can achieve the detection boundary for testing $H_0$ against $H_1(\Delta_n(\theta), \theta)$ simultaneously over all $\theta\ge 0$. We emphasize the dependence of $\Delta_n$ on $\theta$ since the detection boundary may depend on $\theta$, as suggested by the results from the previous section. In fact, we should build upon the test statistic introduced before.

More specifically, write
$$
\rho_\ast=\left(\frac{\sqrt{\log \log n}}{n}\right)^{2s},
$$
and
$$
m_\ast=\left\lceil\log_2 \left[\rho_\ast^{-1}\left(\frac{\sqrt{\log\log n}}{n}\right)^{\frac{2s}{4s+1}}\right]\right\rceil.
$$
Then our test statistics is taken to be the maximum of $T_{n,\varrho_n}$ for $\rho_n=\rho_\ast, 2\rho_\ast, 2^2\rho_\ast,\ldots, 2^{m_\ast}\rho_\ast$:
\begin{align*}
\tilde{T}_n:= \sup\limits_{0\le k\le m_\ast}T_{n,2^k\varrho_\ast},
\end{align*}
where, with slight abuse of notation,
$$T_{n,\varrho_n}=(2v_n)^{-1/2}[n\eta_{\varrho_n}^2(\hat{P}_n,P_0)-A_n].$$ 
It turns out if an appropriate rejection threshold is chosen, $\tilde{T}_n$ can achieve a detection boundary very similar to the one we have before, but now simultaneously over all $\theta>0$.

\begin{theorem}\label{crtm3d}
	\rm(\romannumeral1) \it Under $H_0$,
	\begin{align*}
	\lim\limits_{n\rightarrow \infty}P\left(\tilde{T}_n\geq \sqrt{3\log\log n}\right)=0;
	\end{align*}
	\rm(\romannumeral2) \it on the other hand, there exists a constant $c_1>0$ such that,
	\begin{align*}
	\lim\limits_{n\rightarrow \infty}\inf\limits_{P\in\cup_{\theta\ge 0}\Pcal(\Delta_{n}(\theta))}P\left(\tilde{T}_n\geq \sqrt{3\log\log n}\right)=1,
	\end{align*}
	provided that $\Delta_n(\theta)\geq c_1(n^{-1}{\sqrt{\log\log n}})^{\frac{4s}{4s+\theta+1}}$.
\end{theorem}

Theorem \ref{crtm3d} immediately suggests that a test rejects $H_0$ if and only if $\tilde{T}_n\geq \sqrt{3\log\log n}$ is consistent for testing it against $H_1(\Delta_n(\theta),\theta)$ for all $\theta\ge 0$ provided that $\Delta_n(\theta)\geq c_1(n^{-1}{\sqrt{\log\log n}})^{\frac{4s}{4s+\theta+1}}$. We can further calibrate the rejection region to yield a test at a given significance level. More precisely, let $\tilde{q}_\alpha$ be the upper $\alpha$ quantile of $\tilde{T}_n$, we can proceed to reject $H_0$ whenever the observed test statistic exceeds $\tilde{q}_\alpha$. Denote such a test by \TTMd. By definition, \TTMd\ is an $\alpha$-level test. Theorem \ref{crtm3d} implies that the type II error of \TTMd\ vanishes as $n\to\infty$ uniformly over all $\theta\ge 0$. In practice, the quantile $\tilde{q}_\alpha$ can be evaluated by Monte Carlo methods as we shall discuss in further details in the next section. We note that the detection boundary given in Theorem \ref{crtm3d} is similar, but inferior by a factor of $(\log\log n)^{\frac{2s}{4s+\theta+1}}$, to that from Theorem \ref{cr}. This turns out be the price one needs to pay for adaptation.

\begin{theorem}\label{cra}
Let $0<\theta_1<\theta_2<2s-1$. Then there exists a positive constant $c_2$ such that
\begin{align*}
\varlimsup\limits_{n\rightarrow \infty}\sup\limits_{\theta\in[\theta_1,\theta_2]}\left\{\Delta_n(\theta)\left(\frac{n}{\sqrt{\log\log n}}\right)^{\frac{4s}{4s+\theta+1}}\right\}\leq c_2
\end{align*}
implies that
\begin{align*}
\lim\limits_{n\rightarrow \infty}\inf\limits_{{T}\in \mathcal{T}_n}\left[\EE_{P_0} {T}+\sup_{\theta\in [\theta_1,\theta_2]}\beta({T};\Delta_n(\theta), \theta)\right]=1.
\end{align*}
%
\end{theorem}

Similar to Theorem \ref{cr}, Theorem \ref{cra} shows that there is no consistent test for $H_0$ against $H_1(\Delta_n,\theta)$ simultaneously over all $\theta\in [\theta_1,\theta_2]$, if $\Delta_n(\theta)\leq c_2\left(n^{-1}\sqrt{\log\log n}\right)^{\frac{4s}{4s+\theta+1}}$ $\forall\ \theta\in[\theta_1,\theta_2]$ for a sufficiently small $c_2$. Together with Theorem \ref{crtm3d}, this suggests that the test \TTMd\ is indeed rate optimal.

\section{Numerical Experiments}
\label{sec:num}

To complement the earlier theoretical development, we also performed several sets of simulation experiments to demonstrate the merits of the proposed adaptive test based on $\tilde{T}_n$. To do so, we need to first address a practical issue of computing the test statistic $\tilde{T}_n$:  how to compute ${\eta}_{\varrho_n}^2(\hat{P}_n, P_0)$ for a given $\varrho_n$.

\subsection{Computing $\tilde{T}_n$}\label{sec:compute}
Though the form of ${\eta}_{\varrho_n}^2(\hat{P}_n, P_0)$ looks similar to that of ${\gamma}^2(\hat{P}_n, P_0)$, from the point of view of computing it numerically, there is a subtle issue. The kernel $\tilde{K}_{\varrho_n}(x,x')$ is defined only in its Mercer decomposed form, which is based on the Mercer decomposition of $K(x,x')$. Hence, in order to compute the kernel $\tilde{K}_{\varrho_n}(x,x')$, we need to first choose a kernel $K(x,x')$ and compute its Mercer decomposition numerically. Specifically, we use \emph{chebfun} framework in Matlab (with slight modifications) to compute Mercer decompositions associated with kernels based on their integral operator representations ~\cite{Driscoll2014, Battles2004}. Once we compute $\lambda_k$ and the associated $\varphi_k(\cdot)$, we approximately compute $\tilde{K}_{\varrho_n}(x,x')$ based on the top $K$ eigvevalues and eigenfunctions. This provides a numerical framework for computing $\tilde{K}_{\varrho_n}(x,x')$ once we fix a kernel $K(x,x')$. In the cases when the eigenvalues and eigenfunction are known, for example when using polynomial kernels, from our experiments we found that using the top few numerical eigenvalues gives a good approximation to the actual value of the kernel. Given a way to compute kernel evaluations, computing  ${\eta}_{\varrho_n}^2(\hat{P}_n, P_0)$ follows similarly. 

\subsection{ Power comparison}

Once we are able to compute $\tilde{T}_n$, we can assess its null distribution by simulating it under the null hypothesis $H_0$. In particular, we repeated for each case 200 runs and estimated the $95\%$ quantile of $\tilde{T}_n$ under $H_0$ by the corresponding sample quantile. We then proceeded to reject $H_0$ when an observed test statistic exceeds the estimate $95\%$ quantile. By construction, the procedure gives a $5\%$-level test, up to Monte Carlo error.

\paragraph{Euclidean data:} We first consider using both the test \TMMD\ and \TTMd\ to test the hypothesis $P_0$ is uniform on $[0,1]^d$ given a sample of observations $\{ X_1, \ldots, X_n\}$. The dimensionality used are $100$ and $200$. We followed the examples for densities put forward in~\cite{marron1992exact} in the context of nonparametric density estimation, for the alternatives.  Specifically we set the alternative hypothesis to be (1) mixture of five Gaussians, (2) skewed unimodal, (3) asymmetric claw density and (4) smooth comb density. The value of $\alpha$ is set to $0.05$. The sample size $n$ is varied from $200$ to $1000$ (in steps of 200) and for each value of sample size 100 simulations are conducted to estimate the probability of rejecting a false null hypothesis. 

We use a Gaussian kernel $K$ to compute the ${\gamma}^2(\hat{P}_n, P_0)$ and use the procedure outlined in section~\ref{sec:compute} to compute ${\eta}_{\varrho_n}^2(\hat{P}_n, P_0)$. The issue of choosing the kernel is subtle when using \TMMD. For simplicity, we fixed the value of bandwidth of Gaussian kernel (which corresponds to choosing the kernel in this case) to a fixed value, that corresponds to the best performance of \TMMD. With the fixed value of the bandwidth, to fix $\varrho$, we tried values over a grid and set it to the value that performed best. Figure~\ref{fig:simT2} illustrates a plot of the estimated probability of  accepting the null hypothesis when it is false for different values of sample size $n$ for the proposed test \TTMd\ along with \TMMD\ and the more classical Kolmogorov-Smirnov (K-S, for short) goodness-of-fit test. We note from Figure~\ref{fig:simT2} that the estimated error probability converges to zero at a faster rate for the adaptive M$^3$D test compared to the MMD test and the Kolmogorov-Smirnov test on all the different simulation settings that are considered. Note that it has been previously observed that MMD test performs better than K-S test in various setting in \cite{gretton2012kernel}, which we observe in our setting as well.

\begin{figure*}[!htbp]
\centering
\begin{minipage}{0.5\textwidth}
\centering
\begin{tikzpicture}[scale=0.70]
  \begin{axis}[
    xlabel = $\text{Sample size (n)}$,
ylabel=$P(~\text{accepting}~H_0~\text{when false}~)$,
xmax = 1000,
xmin = 200,
ymax = 1,
ymin = 0
]
\addplot+[error bars/.cd,
y dir=both,y explicit]
 coordinates {
( 200, 0.94 )+- (0.0, 0.04)
( 400, 0.40 )+- (0.0, 0.03)
( 600, 0.31 )+- (0.0, 0.00)
( 800, 0.15 )+- (0.0, 0.00)
( 1000, 0.05 )+- (0.0, 0.00)
};  \addlegendentry{$M^3D$} ;
\addplot+[error bars/.cd,
y dir=both,y explicit]
 coordinates {
( 200, 0.95 )+- (0.0, 0.03)
( 400, 0.45 )+- (0.0, 0.03)
( 600, 0.38 )+- (0.0, 0.03)
( 800, 0.22 )+- (0.0, 0.03)
( 1000, 0.08 )+- (0.0, 0.00)
};  \addlegendentry{$MMD$} ;
\addplot+[error bars/.cd,
y dir=both,y explicit]
 coordinates {
( 200, 0.98 )+- (0.0, 0.03)
( 400, 0.52 )+- (0.0, 0.03)
( 600, 0.44 )+- (0.0, 0.03)
( 800, 0.28 )+- (0.0, 0.03)
( 1000, 0.12 )+- (0.0, 0.03)
};  \addlegendentry{$KS$} ;
\end{axis}
\end{tikzpicture}
\end{minipage}
\begin{minipage}{0.5\textwidth}
\centering
\begin{tikzpicture}[scale=0.70]
  \begin{axis}[
    xlabel = $\text{Sample size (n)}$,
ylabel=$P(~\text{accepting}~H_0~\text{when false}~)$,
xmax = 1000,
xmin = 200,
ymax = 1,
ymin = 0
]
\addplot+[error bars/.cd,
y dir=both,y explicit]
 coordinates {
( 200, 0.94 )+- (0.0, 0.04)
( 400, 0.40 )+- (0.0, 0.03)
( 600, 0.30 )+- (0.0, 0.00)
( 800, 0.20 )+- (0.0, 0.00)
( 1000, 0.06 )+- (0.0, 0.00)
};  \addlegendentry{$M^3D$} ;
\addplot+[error bars/.cd,
y dir=both,y explicit]
 coordinates {
( 200, 0.95 )+- (0.0, 0.03)
( 400, 0.44 )+- (0.0, 0.03)
( 600, 0.35 )+- (0.0, 0.03)
( 800, 0.23 )+- (0.0, 0.03)
( 1000, 0.09 )+- (0.0, 0.00)
};  \addlegendentry{$MMD$} ;
\addplot+[error bars/.cd,
y dir=both,y explicit]
 coordinates {
( 200, 0.96 )+- (0.0, 0.03)
( 400, 0.53 )+- (0.0, 0.03)
( 600, 0.44 )+- (0.0, 0.03)
( 800, 0.30 )+- (0.0, 0.00)
( 1000, 0.13 )+- (0.0, 0.00)
};  \addlegendentry{$KS$} ;
\end{axis}
\end{tikzpicture}
\end{minipage}
\hfill \vspace{-0.1in}
\begin{minipage}{0.5\textwidth}
\centering
\begin{tikzpicture}[scale=0.70]
  \begin{axis}[
    xlabel = $\text{Sample size (n)}$,
ylabel=$P(~\text{accepting}~H_0~\text{when false}~)$,
xmax = 1000,
xmin = 200,
ymax = 1,
ymin = 0
]
\addplot+[error bars/.cd,
y dir=both,y explicit]
 coordinates {
( 200, 0.95 )+- (0.0, 0.04)
( 400, 0.30 )+- (0.0, 0.03)
( 600, 0.14 )+- (0.0, 0.00)
( 800, 0.105 )+- (0.0, 0.00)
( 1000, 0.06 )+- (0.0, 0.00)
};  \addlegendentry{$M^3D$} ;
\addplot+[error bars/.cd,
y dir=both,y explicit]
 coordinates {
( 200, 0.95 )+- (0.0, 0.03)
( 400, 0.46 )+- (0.0, 0.03)
( 600, 0.35 )+- (0.0, 0.03)
( 800, 0.19 )+- (0.0, 0.00)
( 1000, 0.10 )+- (0.0, 0.00)
};  \addlegendentry{$MMD$} ;
\addplot+[error bars/.cd,
y dir=both,y explicit]
 coordinates {
( 200, 0.96 )+- (0.0, 0.03)
( 400, 0.59 )+- (0.0, 0.03)
( 600, 0.48 )+- (0.0, 0.03)
( 800, 0.23 )+- (0.0, 0.00)
( 1000, 0.11 )+- (0.0, 0.00)
};  \addlegendentry{$KS$} ;
\end{axis}
\end{tikzpicture}
\end{minipage}
\begin{minipage}{0.5\textwidth}
\centering
\begin{tikzpicture}[scale=0.70]
  \begin{axis}[
    xlabel = $\text{Sample size (n)}$,
ylabel=$P(~\text{accepting}~H_0~\text{when false}~)$,
xmax = 1000,
xmin = 200,
ymax = 1,
ymin = 0
]
\addplot+[error bars/.cd,
y dir=both,y explicit]
 coordinates {
( 200, 0.95 )+- (0.0, 0.04)
( 400, 0.32 )+- (0.0, 0.03)
( 600, 0.19 )+- (0.0, 0.00)
( 800, 0.11 )+- (0.0, 0.00)
( 1000, 0.05 )+- (0.0, 0.00)
};  \addlegendentry{$M^3D$} ;
\addplot+[error bars/.cd,
y dir=both,y explicit]
 coordinates {
( 200, 0.95 )+- (0.0, 0.03)
( 400, 0.46)+- (0.0, 0.03)
( 600, 0.26 )+- (0.0, 0.03)
( 800, 0.20 )+- (0.0, 0.03)
( 1000, 0.09 )+- (0.0, 0.00)
};  \addlegendentry{$MMD$} ;
\addplot+[error bars/.cd,
y dir=both,y explicit]
 coordinates {
( 200, 0.96 )+- (0.0, 0.03)
( 400, 0.59 )+- (0.0, 0.03)
( 600, 0.41 )+- (0.0, 0.03)
( 800, 0.30 )+- (0.0, 0.00)
( 1000, 0.11 )+- (0.0, 0.00)
};  \addlegendentry{$KS$} ;
\end{axis}
\end{tikzpicture}
\end{minipage}
\vspace{-0.1in}
\begin{minipage}{0.5\textwidth}
\centering
\begin{tikzpicture}[scale=0.70]
  \begin{axis}[
    xlabel = $\text{Sample size (n)}$,
ylabel=$P(~\text{accepting}~H_0~\text{when false}~)$,
xmax = 1000,
xmin = 200,
ymax = 1,
ymin = 0
]
\addplot+[error bars/.cd,
y dir=both,y explicit]
 coordinates {
( 200, 0.96 )+- (0.0, 0.04)
( 400, 0.35 )+- (0.0, 0.03)
( 600, 0.20 )+- (0.0, 0.00)
( 800, 0.15 )+- (0.0, 0.00)
( 1000, 0.05 )+- (0.0, 0.00)
};  \addlegendentry{$M^3D$} ;
\addplot+[error bars/.cd,
y dir=both,y explicit]
 coordinates {
( 200, 0.96 )+- (0.0, 0.03)
( 400, 0.52 )+- (0.0, 0.03)
( 600, 0.41 )+- (0.0, 0.03)
( 800, 0.25 )+- (0.0, 0.03)
( 1000, 0.10 )+- (0.0, 0.00)
};  \addlegendentry{$MMD$} ;
\addplot+[error bars/.cd,
y dir=both,y explicit]
 coordinates {
( 200, 0.97 )+- (0.0, 0.03)
( 400, 0.63 )+- (0.0, 0.03)
( 600, 0.51 )+- (0.0, 0.03)
( 800, 0.35 )+- (0.0, 0.00)
( 1000, 0.10 )+- (0.0, 0.00)
};  \addlegendentry{$KS$} ;
\end{axis}
\end{tikzpicture}
\end{minipage}
\begin{minipage}{0.5\textwidth}
\centering
\begin{tikzpicture}[scale=0.70]
  \begin{axis}[
    xlabel = $\text{Sample size (n)}$,
ylabel=$P(~\text{accepting}~H_0~\text{when false}~)$,
xmax = 1000,
xmin = 200,
ymax = 1,
ymin = 0
]
\addplot+[error bars/.cd,
y dir=both,y explicit]
 coordinates {
( 200, 0.96 )+- (0.0, 0.04)
( 400, 0.48 )+- (0.0, 0.03)
( 600, 0.30 )+- (0.0, 0.00)
( 800, 0.15 )+- (0.0, 0.00)
( 1000, 0.06 )+- (0.0, 0.00)
};  \addlegendentry{$M^3D$} ;
\addplot+[error bars/.cd,
y dir=both,y explicit]
 coordinates {
( 200, 0.97 )+- (0.0, 0.03)
( 400, 0.66 )+- (0.0, 0.03)
( 600, 0.45 )+- (0.0, 0.03)
( 800, 0.26 )+- (0.0, 0.03)
( 1000, 0.09 )+- (0.0, 0.00)
};  \addlegendentry{$MMD$} ;
\addplot+[error bars/.cd,
y dir=both,y explicit]
 coordinates {
( 200, 0.97 )+- (0.0, 0.03)
( 400, 0.69 )+- (0.0, 0.03)
( 600, 0.51 )+- (0.0, 0.03)
( 800, 0.31 )+- (0.0, 0.00)
( 1000, 0.104 )+- (0.0, 0.00)
};  \addlegendentry{$KS$} ;
\end{axis}
\end{tikzpicture}
\end{minipage}
\vspace{-0.1in}
\begin{minipage}{0.5\textwidth}
\centering
\begin{tikzpicture}[scale=0.70]
  \begin{axis}[
    xlabel = $\text{Sample size (n)}$,
ylabel=$P(~\text{accepting}~H_0~\text{when false}~)$,
xmax = 1000,
xmin = 200,
ymax = 1,
ymin = 0
]
\addplot+[error bars/.cd,
y dir=both,y explicit]
 coordinates {
( 200, 0.96 )+- (0.0, 0.04)
( 400, 0.41 )+- (0.0, 0.03)
( 600, 0.25 )+- (0.0, 0.00)
( 800, 0.16 )+- (0.0, 0.00)
( 1000, 0.04 )+- (0.0, 0.00)
};  \addlegendentry{$M^3D$} ;
\addplot+[error bars/.cd,
y dir=both,y explicit]
 coordinates {
( 200, 0.97 )+- (0.0, 0.03)
( 400, 0.54 )+- (0.0, 0.03)
( 600, 0.40 )+- (0.0, 0.03)
( 800, 0.22 )+- (0.0, 0.03)
( 1000, 0.06 )+- (0.0, 0.00)
};  \addlegendentry{$MMD$} ;
\addplot+[error bars/.cd,
y dir=both,y explicit]
 coordinates {
( 200, 0.97 )+- (0.0, 0.03)
( 400, 0.60 )+- (0.0, 0.03)
( 600, 0.45 )+- (0.0, 0.04)
( 800, 0.33 )+- (0.0, 0.00)
( 1000, 0.10 )+- (0.0, 0.00)
};  \addlegendentry{$KS$} ;
\end{axis}
\end{tikzpicture}
\end{minipage}
\begin{minipage}{0.5\textwidth}
\centering
\begin{tikzpicture}[scale=0.70]
  \begin{axis}[
    xlabel = $\text{Sample size (n)}$,
ylabel=$P(~\text{accepting}~H_0~\text{when false}~)$,
xmax = 1000,
xmin = 200,
ymax = 1,
ymin = 0
]
\addplot+[error bars/.cd,
y dir=both,y explicit]
 coordinates {
( 200, 0.96 )+- (0.0, 0.04)
( 400, 0.41 )+- (0.0, 0.03)
( 600, 0.25 )+- (0.0, 0.00)
( 800, 0.16 )+- (0.0, 0.00)
( 1000, 0.04 )+- (0.0, 0.00)
};  \addlegendentry{$M^3D$} ;
\addplot+[error bars/.cd,
y dir=both,y explicit]
 coordinates {
( 200, 0.97 )+- (0.0, 0.03)
( 400, 0.54 )+- (0.0, 0.03)
( 600, 0.40 )+- (0.0, 0.03)
( 800, 0.22 )+- (0.0, 0.03)
( 1000, 0.06 )+- (0.0, 0.00)
};  \addlegendentry{$MMD$} ;
\addplot+[error bars/.cd,
y dir=both,y explicit]
 coordinates {
( 200, 0.97 )+- (0.0, 0.03)
( 400, 0.60 )+- (0.0, 0.03)
( 600, 0.45 )+- (0.0, 0.03)
( 800, 0.33 )+- (0.0, 0.00)
( 1000, 0.10 )+- (0.0, 0.00)
};  \addlegendentry{$KS$} ;
\end{axis}
\end{tikzpicture}
\end{minipage}
\hfill\vspace{-0.1in}
\caption{Error versus Sample Size: mixture of Gaussian (row 1), skewed unimodal (row 2), asymmetric claw (row 3) and smooth comb (row 4) with dimensionality 100 (left) and 200 (right).}
\label{fig:simT2}
\end{figure*}
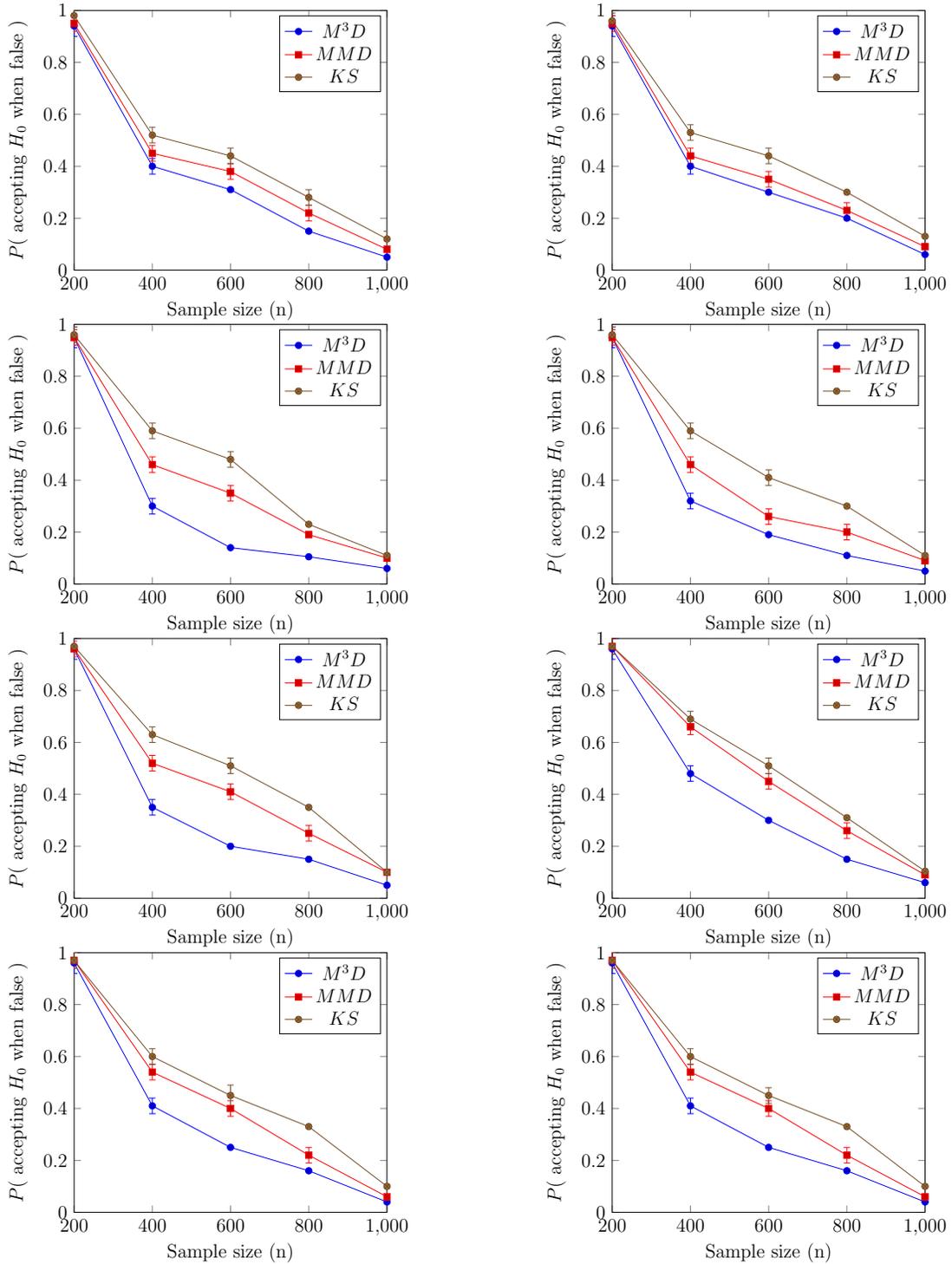


\paragraph{Directional data:} One of the advantages of the proposed RKHS embedding based approach is that it could be used on domains other than the $d$-dimensional Euclidean space. For example when $\mathcal{X} = \mathbb{S}^{d-1}$ where  $\mathbb{S}^{d-1}$ corresponds to the $d$-dimensional unit sphere, one can perform hypothesis testing using the above framework, as long as we can compute the Mercer decomposition of a kernel $K$ defined on the domain. In several applications, like protein folding, often times data are modeled as coming from the unit-sphere and testing goodness-of-fit for such data needs specialized methods different from the standard nonparametric testing methods~\cite{mardia2009directional,jupp2005sobolev}. 

In order to highlight the advantage of the proposed approach, we assume $P_0$ is uniform distribution on the unit sphere of dimension $100$ and test it against the alternative that data are from: 
\begin{enumerate}
\item[(1)] multivariate von Mises-Fisher distribution (which is the Gaussian analogue on the unit-sphere) given by $f_{vM\mhyphen F}(x, \mu, \kappa) = C_{vM\mhyphen F}(\kappa) \exp(\kappa \mu^\top x)$ for data $x \in \mathbb{S}^{d-1}$, where $\kappa \geq 0$ is concentration parameter and $\mu$ is the mean parameter. The term $C_{vM\mhyphen F}$ is the normalization constant given by $\frac{\kappa^{d/2-1}}{2\pi^{d/2}I_{d/2-1}(\kappa)}$ where $I$ is modified Bessel function;

\item[(2)] multivariate Watson distribution (used to model axially symmetric data on sphere) given by  $f_{W}(x, \mu, \kappa) = C_W(\kappa) \exp(\kappa (\mu^\top x)^2)$ for data $x \in \mathbb{S}^{d-1}$, where $\kappa \geq 0$ is concentration parameter and $\mu$ is the mean parameter as before. The term $C_W(\kappa)$ is the normalization constant given by $\frac{\Gamma(d/2)}{2\pi^{p/2} M(1/2,d/2,\kappa)}$ where $M$ is Kummer's confluent hypergeometric function;

\item[(3)] mixture of five von Mises-Fisher distribution which are used in modeling and clustering spherical data~\cite{banerjee2005clustering};
\item[(4)] mixture of five Watson distribution which are used in modeling and clustering spherical data~\cite{sra2013multivariate}.
\end{enumerate}

Note that in this setup one can analytically compute the Mercer decomposition of the Gaussian kernel on the unit sphere with respect to the uniform distribution. Specifically, the eigenvalues are given by Theorem 2 in~\cite{minh2006mercer} and the eigenfunctions are the standard spherical harmonics of order $k$ (see section 2.1 in ~\cite{minh2006mercer} for details). Rest of the simulation setup is similar to the previous setting (of Euclidean data) and we compared \TTMd\ against  \TMMD\ and the Sobolev test approach (denoted as ST hereafter) proposed in~\cite{jupp2005sobolev}. Figure~\ref{fig:simT3} illustrates a plot of estimated probability of accepting null hypothesis when it is false for different values of sample size, from which we see the adaptive M$^3$D test performs better.  

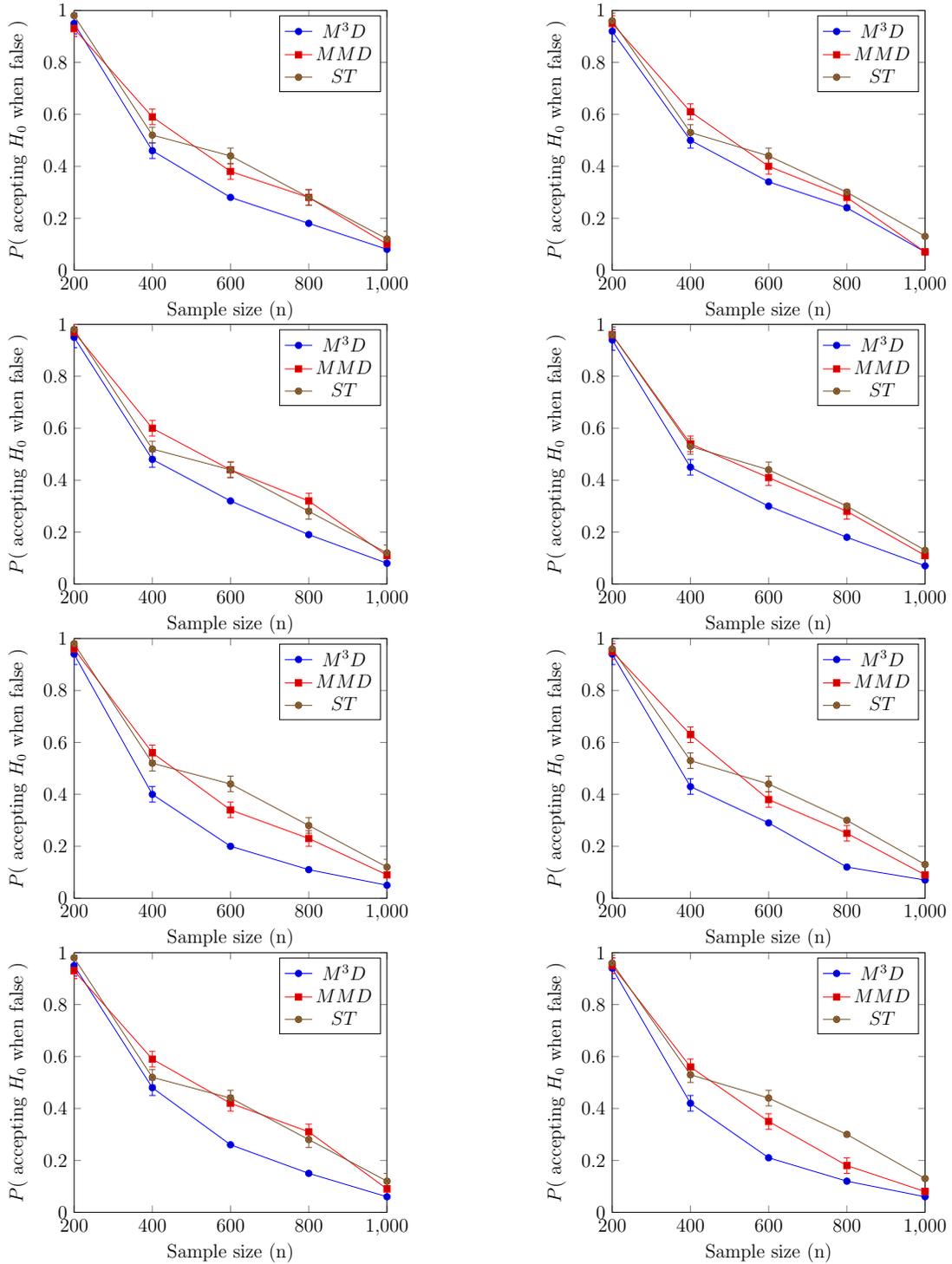
\begin{figure*}[!htbp]
\centering
\begin{minipage}{0.5\textwidth}
\centering
\begin{tikzpicture}[scale=0.70]
  \begin{axis}[
    xlabel = $\text{Sample size (n)}$,
ylabel=$P(~\text{accepting}~H_0~\text{when false}~)$,
xmax = 1000,
xmin = 200,
ymax = 1,
ymin = 0
]
\addplot+[error bars/.cd,
y dir=both,y explicit]
 coordinates {
( 200, 0.95 )+- (0.0, 0.04)
( 400, 0.46 )+- (0.0, 0.03)
( 600, 0.28 )+- (0.0, 0.00)
( 800, 0.18 )+- (0.0, 0.00)
( 1000, 0.08 )+- (0.0, 0.00)
};  \addlegendentry{$M^3D$} ;
\addplot+[error bars/.cd,
y dir=both,y explicit]
 coordinates {
( 200, 0.93 )+- (0.0, 0.03)
( 400, 0.59 )+- (0.0, 0.03)
( 600, 0.38 )+- (0.0, 0.03)
( 800, 0.28 )+- (0.0, 0.03)
( 1000, 0.10 )+- (0.0, 0.00)
};  \addlegendentry{$MMD$} ;
\addplot+[error bars/.cd,
y dir=both,y explicit]
 coordinates {
( 200, 0.98 )+- (0.0, 0.03)
( 400, 0.52 )+- (0.0, 0.03)
( 600, 0.44 )+- (0.0, 0.03)
( 800, 0.28 )+- (0.0, 0.03)
( 1000, 0.12 )+- (0.0, 0.03)
};  \addlegendentry{$ST$} ;
\end{axis}
\end{tikzpicture}
\end{minipage}
\begin{minipage}{0.5\textwidth}
\centering
\begin{tikzpicture}[scale=0.70]
  \begin{axis}[
    xlabel = $\text{Sample size (n)}$,
ylabel=$P(~\text{accepting}~H_0~\text{when false}~)$,
xmax = 1000,
xmin = 200,
ymax = 1,
ymin = 0
]
\addplot+[error bars/.cd,
y dir=both,y explicit]
 coordinates {
( 200, 0.92 )+- (0.0, 0.04)
( 400, 0.50 )+- (0.0, 0.03)
( 600, 0.34 )+- (0.0, 0.00)
( 800, 0.24 )+- (0.0, 0.00)
( 1000, 0.07 )+- (0.0, 0.00)
};  \addlegendentry{$M^3D$} ;
\addplot+[error bars/.cd,
y dir=both,y explicit]
 coordinates {
( 200, 0.95 )+- (0.0, 0.03)
( 400, 0.61 )+- (0.0, 0.03)
( 600, 0.40 )+- (0.0, 0.03)
( 800, 0.28 )+- (0.0, 0.03)
( 1000, 0.07 )+- (0.0, 0.00)
};  \addlegendentry{$MMD$} ;
\addplot+[error bars/.cd,
y dir=both,y explicit]
 coordinates {
( 200, 0.96 )+- (0.0, 0.03)
( 400, 0.53 )+- (0.0, 0.03)
( 600, 0.44 )+- (0.0, 0.03)
( 800, 0.30 )+- (0.0, 0.00)
( 1000, 0.13 )+- (0.0, 0.00)
};  \addlegendentry{$ST$} ;
\end{axis}
\end{tikzpicture}
\end{minipage}
\hfill \vspace{-0.1in}
\begin{minipage}{0.5\textwidth}
\centering
\begin{tikzpicture}[scale=0.70]
  \begin{axis}[
    xlabel = $\text{Sample size (n)}$,
ylabel=$P(~\text{accepting}~H_0~\text{when false}~)$,
xmax = 1000,
xmin = 200,
ymax = 1,
ymin = 0
]
\addplot+[error bars/.cd,
y dir=both,y explicit]
 coordinates {
( 200, 0.95 )+- (0.0, 0.04)
( 400, 0.48 )+- (0.0, 0.03)
( 600, 0.32 )+- (0.0, 0.00)
( 800, 0.19 )+- (0.0, 0.00)
( 1000, 0.08 )+- (0.0, 0.00)
};  \addlegendentry{$M^3D$} ;
\addplot+[error bars/.cd,
y dir=both,y explicit]
 coordinates {
( 200, 0.97 )+- (0.0, 0.03)
( 400, 0.60 )+- (0.0, 0.03)
( 600, 0.44 )+- (0.0, 0.03)
( 800, 0.32 )+- (0.0, 0.03)
( 1000, 0.11 )+- (0.0, 0.00)
};  \addlegendentry{$MMD$} ;
\addplot+[error bars/.cd,
y dir=both,y explicit]
 coordinates {
( 200, 0.98 )+- (0.0, 0.03)
( 400, 0.52 )+- (0.0, 0.03)
( 600, 0.44 )+- (0.0, 0.03)
( 800, 0.28 )+- (0.0, 0.03)
( 1000, 0.12 )+- (0.0, 0.03)
};  \addlegendentry{$ST$} ;
\end{axis}
\end{tikzpicture}
\end{minipage}
\begin{minipage}{0.5\textwidth}
\centering
\begin{tikzpicture}[scale=0.70]
  \begin{axis}[
    xlabel = $\text{Sample size (n)}$,
ylabel=$P(~\text{accepting}~H_0~\text{when false}~)$,
xmax = 1000,
xmin = 200,
ymax = 1,
ymin = 0
]
\addplot+[error bars/.cd,
y dir=both,y explicit]
 coordinates {
( 200, 0.94 )+- (0.0, 0.04)
( 400, 0.45 )+- (0.0, 0.03)
( 600, 0.30 )+- (0.0, 0.00)
( 800, 0.18 )+- (0.0, 0.00)
( 1000, 0.07 )+- (0.0, 0.00)
};  \addlegendentry{$M^3D$} ;
\addplot+[error bars/.cd,
y dir=both,y explicit]
 coordinates {
( 200, 0.96 )+- (0.0, 0.03)
( 400, 0.54 )+- (0.0, 0.03)
( 600, 0.41 )+- (0.0, 0.03)
( 800, 0.28 )+- (0.0, 0.03)
( 1000, 0.11 )+- (0.0, 0.00)
};  \addlegendentry{$MMD$} ;
\addplot+[error bars/.cd,
y dir=both,y explicit]
 coordinates {
( 200, 0.96 )+- (0.0, 0.03)
( 400, 0.53 )+- (0.0, 0.03)
( 600, 0.44 )+- (0.0, 0.03)
( 800, 0.30 )+- (0.0, 0.00)
( 1000, 0.13 )+- (0.0, 0.00)
};  \addlegendentry{$ST$} ;
\end{axis}
\end{tikzpicture}
\end{minipage}
\vspace{-0.1in}
\begin{minipage}{0.5\textwidth}
\centering
\begin{tikzpicture}[scale=0.70]
  \begin{axis}[
    xlabel = $\text{Sample size (n)}$,
ylabel=$P(~\text{accepting}~H_0~\text{when false}~)$,
xmax = 1000,
xmin = 200,
ymax = 1,
ymin = 0
]
\addplot+[error bars/.cd,
y dir=both,y explicit]
 coordinates {
( 200, 0.94 )+- (0.0, 0.04)
( 400, 0.40 )+- (0.0, 0.03)
( 600, 0.20 )+- (0.0, 0.00)
( 800, 0.11 )+- (0.0, 0.00)
( 1000, 0.05 )+- (0.0, 0.00)
};  \addlegendentry{$M^3D$} ;
\addplot+[error bars/.cd,
y dir=both,y explicit]
 coordinates {
( 200, 0.96 )+- (0.0, 0.03)
( 400, 0.56 )+- (0.0, 0.03)
( 600, 0.34 )+- (0.0, 0.03)
( 800, 0.23 )+- (0.0, 0.03)
( 1000, 0.09 )+- (0.0, 0.00)
};  \addlegendentry{$MMD$} ;
\addplot+[error bars/.cd,
y dir=both,y explicit]
 coordinates {
( 200, 0.98 )+- (0.0, 0.03)
( 400, 0.52 )+- (0.0, 0.03)
( 600, 0.44 )+- (0.0, 0.03)
( 800, 0.28 )+- (0.0, 0.03)
( 1000, 0.12 )+- (0.0, 0.03)
};  \addlegendentry{$ST$} ;
\end{axis}
\end{tikzpicture}
\end{minipage}
\begin{minipage}{0.5\textwidth}
\centering
\begin{tikzpicture}[scale=0.70]
  \begin{axis}[
    xlabel = $\text{Sample size (n)}$,
ylabel=$P(~\text{accepting}~H_0~\text{when false}~)$,
xmax = 1000,
xmin = 200,
ymax = 1,
ymin = 0
]
\addplot+[error bars/.cd,
y dir=both,y explicit]
 coordinates {
( 200, 0.94 )+- (0.0, 0.04)
( 400, 0.43 )+- (0.0, 0.03)
( 600, 0.29 )+- (0.0, 0.00)
( 800, 0.12 )+- (0.0, 0.00)
( 1000, 0.07 )+- (0.0, 0.00)
};  \addlegendentry{$M^3D$} ;
\addplot+[error bars/.cd,
y dir=both,y explicit]
 coordinates {
( 200, 0.95 )+- (0.0, 0.03)
( 400, 0.63 )+- (0.0, 0.03)
( 600, 0.38 )+- (0.0, 0.03)
( 800, 0.25 )+- (0.0, 0.03)
( 1000, 0.09 )+- (0.0, 0.00)
};  \addlegendentry{$MMD$} ;
\addplot+[error bars/.cd,
y dir=both,y explicit]
 coordinates {
( 200, 0.96 )+- (0.0, 0.03)
( 400, 0.53 )+- (0.0, 0.03)
( 600, 0.44 )+- (0.0, 0.03)
( 800, 0.30 )+- (0.0, 0.00)
( 1000, 0.13 )+- (0.0, 0.00)
};  \addlegendentry{$ST$} ;
\end{axis}
\end{tikzpicture}
\end{minipage}
\vspace{-0.1in}
\begin{minipage}{0.5\textwidth}
\centering
\begin{tikzpicture}[scale=0.70]
  \begin{axis}[
    xlabel = $\text{Sample size (n)}$,
ylabel=$P(~\text{accepting}~H_0~\text{when false}~)$,
xmax = 1000,
xmin = 200,
ymax = 1,
ymin = 0
]
\addplot+[error bars/.cd,
y dir=both,y explicit]
 coordinates {
( 200, 0.95 )+- (0.0, 0.04)
( 400, 0.48 )+- (0.0, 0.03)
( 600, 0.26 )+- (0.0, 0.00)
( 800, 0.15 )+- (0.0, 0.00)
( 1000, 0.06 )+- (0.0, 0.00)
};  \addlegendentry{$M^3D$} ;
\addplot+[error bars/.cd,
y dir=both,y explicit]
 coordinates {
( 200, 0.93 )+- (0.0, 0.03)
( 400, 0.59 )+- (0.0, 0.03)
( 600, 0.42 )+- (0.0, 0.03)
( 800, 0.31 )+- (0.0, 0.03)
( 1000, 0.09 )+- (0.0, 0.00)
};  \addlegendentry{$MMD$} ;
\addplot+[error bars/.cd,
y dir=both,y explicit]
 coordinates {
( 200, 0.98 )+- (0.0, 0.03)
( 400, 0.52 )+- (0.0, 0.03)
( 600, 0.44 )+- (0.0, 0.03)
( 800, 0.28 )+- (0.0, 0.03)
( 1000, 0.12 )+- (0.0, 0.03)
};  \addlegendentry{$ST$} ;
\end{axis}
\end{tikzpicture}
\end{minipage}
\begin{minipage}{0.5\textwidth}
\centering
\begin{tikzpicture}[scale=0.70]
  \begin{axis}[
    xlabel = $\text{Sample size (n)}$,
ylabel=$P(~\text{accepting}~H_0~\text{when false}~)$,
xmax = 1000,
xmin = 200,
ymax = 1,
ymin = 0
]
\addplot+[error bars/.cd,
y dir=both,y explicit]
 coordinates {
( 200, 0.94 )+- (0.0, 0.04)
( 400, 0.42 )+- (0.0, 0.03)
( 600, 0.21 )+- (0.0, 0.00)
( 800, 0.12 )+- (0.0, 0.00)
( 1000, 0.06 )+- (0.0, 0.00)
};  \addlegendentry{$M^3D$} ;
\addplot+[error bars/.cd,
y dir=both,y explicit]
 coordinates {
( 200, 0.95 )+- (0.0, 0.03)
( 400, 0.56 )+- (0.0, 0.03)
( 600, 0.35 )+- (0.0, 0.03)
( 800, 0.18 )+- (0.0, 0.03)
( 1000, 0.08 )+- (0.0, 0.00)
};  \addlegendentry{$MMD$} ;
\addplot+[error bars/.cd,
y dir=both,y explicit]
 coordinates {
( 200, 0.96 )+- (0.0, 0.03)
( 400, 0.53 )+- (0.0, 0.03)
( 600, 0.44 )+- (0.0, 0.03)
( 800, 0.30 )+- (0.0, 0.00)
( 1000, 0.13 )+- (0.0, 0.00)
};  \addlegendentry{$ST$} ;
\end{axis}
\end{tikzpicture}
\end{minipage}
\hfill\vspace{-0.1in}
\caption{Error versus Sample Size: von Mises-Fisher distribution (row 1), Watson distribution (row 2), Mixture of von Mises-Fisher distribution (row 3) and mixture of Watson distribution (row 4) on sphere for 100 dimensions (left) and 150 dimensions (right).}
\label{fig:simT3}
\end{figure*}


\subsection{Real data experiments}
In addition to the simulation examples, we also performed experiments on several real-world data examples. Similar to before, in the case of Euclidean data, we compared \TTMd\ against the standard MMD and Kolmogorov-Smirnov test, and in the case of spherical data, we used the Sobolev test~\cite{jupp2005sobolev}, instead of the Kolmogorov-Smirnov test.

For the case of Euclidean data, we used the MINST digits data set from the following webpage: \hyperref[http://yann.lecun.com/exdb/mnist/]{http://yann.lecun.com/exdb/mnist/}. Model-based clustering~ \cite{fraley2002model} is a widely-used and practical successful clustering technique in the literature.  Furthermore, the MNIST data set is a standard data set for testing clustering algorithms and consists of image of digits. Several works have implicitly assumed that the data come from a mixture of Gaussian distributions, because of the observed superior empirical performance under such an assumption. But the validity of such a mixture model assumption is invariably not tested statistically. In this experiment we selected three digits (which correspond to a cluster) randomly and conditioned on the selected digit (cluster), we test the hypothesis that the data come from a Gaussian distribution (that is, $P_0$ is Gaussian). For our experiments, we down sampled the images and use pixels as feature vectors with dimensionality 64 as is commonly done in the literature. Table~\ref{tab:realdata1} reports the probability with which the null hypothesis is accepted. The observed result reiterates in a statistically significant way that it is reasonable to make a mixture of Gaussian assumption in this case.

For the spherical case, we use the Human Fibroblasts dataset from~\cite{iyer1999transcriptional,dhillon2003diametrical}, Yeast Cell Cycle dataset from~\cite{spellman1998comprehensive} and the Rosetta yeast gene expression dataset~\cite{hughes2000functional}.  The Fibroblast data set contains 12 expression data corresponding to 517 samples (genes) report in the response of human fibroblasts following addition of serum to the growth media. We refer to~\cite{iyer1999transcriptional} for more details about the scientific procedure with which these data were obtained. The Yeast Cell Cycle dataset consists of 82-dimensional data corresponding to 696 subjects. The Rosetta yeast dataset contains 300-dimensional element vector for around 6000 yeast genes. Previous data analysis studies~\cite{sra2013multivariate, dhillon2003diametrical} have used mixtures of spherical distributions for clustering the above data set. Specifically, it has been observed in~\cite{sra2013multivariate} that clustering using a mixture of Watson distribution has superior performance. While that has proved to be useful scientifically, it was not statistically tested if such an assumption is valid. Here, we test for goodness of fit of Watson distribution (that is, $P_0$ is a Watson distribution) for the largest cluster from the above data sets.  Table~\ref{tab:realdata2} shows the estimated probability of acceptance of the null hypothesis when it is assumed to be true. The values reported are averages from 50 random trails of the same dataset. The observed results provide a statistical justification for the use of Watson distribution in modeling the above data sets. 

We note that for both situations, the tests considered tend to agree that the true hypothesis is true when there are more samples as indicated by Table~\ref{tab:realdata1} and~\ref{tab:realdata2}. But, the probability of acceptance is higher for low sample sizes for the adaptive M$^3$D test \TTMd\, in all cases, showing that the method works better with finite sample sizes. This highlights the advantage of the \TTMd\ in a finite sample setting confirming the better rates of convergence obtained in theory.
\begin{table}[t]
\centering
\begin{tabular}{|c|c|c|c|}
\hline
 Sample size = &  $300$ & $400$& $500$\\ \hline \hline 
K-S &  0.86&0.91 &0.94\\ \hline
$MMD$ & 0.90 & 0.93 &0.95 \\ \hline
$M^3D$& 0.94  &0.96   & 0.98\\ \hline
\end{tabular}
\begin{tabular}{|c|c|c|}
\hline
 $300$ & $400$& $500$\\ \hline \hline 
  0.83& 0.88 & 0.93\\ \hline
 0.89  &0.92  & 0.95\\ \hline
    0.93& 0.95  & 0.98\\ \hline
\end{tabular}
\begin{tabular}{|c|c|c|}
\hline
 $300$ & $400$& $500$\\ \hline \hline 
  0.84& 0.88 & 0.92\\ \hline
 0.88  &0.92  & 0.94\\ \hline
    0.93& 0.95  & 0.98\\ \hline
\end{tabular}
\caption{ The values reported are the estimated probability with which the corresponding hypothesis test accepts the null hypothesis when it is true. The level of the test $\alpha=0.05$.  Digit 4 on left, Digit 6 on the middle and Digit 7 on right, for various values of sample size.}
\label{tab:realdata1}
\end{table}

\begin{table}[t]
\centering
\begin{tabular}{|c|c|c|c|}
\hline
 Sample size = &  $75$ & $150$& $200$\\ \hline \hline 
ST &  0.87&0.93 &0.98\\ \hline
$MMD$ & 0.90 & 0.94 &0.98 \\ \hline
$M^3D$& 0.92  &0.96   & 0.99\\ \hline
\end{tabular}
\begin{tabular}{|c|c|c|}
\hline
  $150$ & $200$& $250$\\ \hline \hline 
  0.82& 0.87 & 0.91\\ \hline
 0.85  &0.92  & 0.94\\ \hline
    0.88& 0.93  & 0.96\\ \hline
\end{tabular}
\begin{tabular}{|c|c|c|}
\hline
 $400$ & $500$& $600$\\ \hline \hline 
  0.76& 0.84 & 0.91\\ \hline
 0.79  &0.87  & 0.92\\ \hline
    0.81& 0.90  & 0.95\\ \hline
\end{tabular}
\caption{ The values reported are the estimated probability with which the corresponding hypothesis test accepts the null hypothesis when it is true. The level of the test $\alpha=0.05$.  Human Fibroblasts dataset  on left, Yeast Cell Cycle dataset on the middle and Rosetta Yeast dataset on the right, for various values of sample size.}
\label{tab:realdata2}
\end{table}


\section{Proofs}
\label{sec:proof}

\begin{proof}[Proof of Theorem \ref{mmdthm}]

\noindent{\bf Part (i).} The proof of the first part consists of two key steps. First, we show that the population counterpart $n\gamma^2(P,P_0)$ of the test statistic converges to $\infty$ uniformly, \ie,
$$n\inf\limits_{P\in\mathcal{P}(\Delta_n,0)}\gamma^2(P,P_0)\rightarrow \infty.$$
Then, we argue that the deviation from $\gamma^2(P,P_0)$ to $\gamma^2(\hat{P}_n,P_0)$ is uniformly negligible compared with $\gamma^2(P,P_0)$ itself.

It is not hard to see that
	\begin{align*}
	\gamma(\hat{P}_n,P_0)=&\sqrt{\sum\limits_{k\geq 1}\lambda_k\Big[\frac{1}{n}\sum\limits_{i=1}^{n}\varphi_k(X_i)\Big]^2}\\
	\geq&\sqrt{\sum\limits_{k\geq 1}\lambda_k[\EE _{P}\varphi_k(X)]^2}-\sqrt{\sum\limits_{k\geq 1}\lambda_k\Big[\frac{1}{n}\sum\limits_{i=1}^{n}\varphi_k(X_i)-\EE _P\varphi_k(X)\Big]^2}.
	\end{align*}
Thus,
	\begin{align*}
	&P\left\{n\gamma^2(\hat{P}_n,P_0)<q_{w,1-\alpha}\right\}\\\leq &P\left\{ \sqrt{n\sum\limits_{k\geq 1}\lambda_k[\EE _{P}\varphi_k(X)]^2}-\sqrt{n\sum\limits_{k\geq 1}\lambda_k\Big[\frac{1}{n}\sum\limits_{i=1}^{n}\varphi_k(X_i)-\EE _P\varphi_k(X)\Big]^2}<\sqrt{q_{w,1-\alpha}}\right\} \\
	=&P\left\{ \sqrt{n\sum\limits_{k\geq 1}\lambda_k\Big[\frac{1}{n}\sum\limits_{i=1}^{n}\varphi_k(X_i)-\EE _P\varphi_k(X)\Big]^2}>\sqrt{n\sum\limits_{k\geq 1}\lambda_k[\EE _{P}\varphi_k(X)]^2}-\sqrt{q_{w,1-\alpha}}\right\}.
	\end{align*}
	
	Suppose that 
	\begin{align*}
	n\sum\limits_{k\geq 1}\lambda_k[\EE _{P}\varphi_k(X)]^2>q_{w,1-\alpha}.
	\end{align*}
Then
	\begin{align*}
	P\left\{n\gamma^2(\hat{P}_n,P_0)<q_{w,1-\alpha}\right\}\leq \frac{\EE _P\Big\{n\sum\limits_{k\geq 1}\lambda_k\Big[\frac{1}{n}\sum\limits_{i=1}^{n}\varphi_k(X_i)-\EE _P\varphi_k(X)\Big]^2\Big\}}{\Bigg\{\sqrt{n\sum\limits_{k\geq 1}\lambda_k[\EE _{P}\varphi_k(X)]^2}-\sqrt{q_{w,1-\alpha}}\Bigg\}^2}.
	\end{align*}
Observe that for any $P\in\mathcal{P}(\Delta_n,0)$,
	\begin{align*}
	\EE _P\Big\{n\sum\limits_{k\geq 1}\lambda_k\Big[\frac{1}{n}\sum\limits_{i=1}^{n}\varphi_k(X_i)-\EE _P\varphi_k(X)\Big]^2\Big\}=&\sum\limits_{k\geq 1}\lambda_k\mathrm{Var}[\varphi_k(X)]\\
	\leq& \sum\limits_{k\geq 1}\lambda_k\EE _P\varphi_k^2(X)\\
	\leq& \Big(\sup\limits_{k\geq 1}\|\varphi_k\|_{\infty}\Big)^2\sum\limits_{k\geq 1}\lambda_k<\infty.
	\end{align*}
This implies that
	\begin{align*}
	\lim\limits_{n\rightarrow \infty}\beta(T_{\text{MMD}};\Delta_n,0)=&\lim\limits_{n\rightarrow \infty}\sup\limits_{P\in\mathcal{P}(\Delta_n,0)}P\left\{n\gamma^2(\hat{P}_n,P_0)<q_{w,1-\alpha}\right\}\\
	\leq&\lim_{n\rightarrow \infty}\frac{\sup\limits_{P\in\mathcal{P}(\Delta_n,0)}\EE _P\Big\{n\sum\limits_{k\geq 1}\lambda_k\Big[\frac{1}{n}\sum\limits_{i=1}^{n}\varphi_k(X_i)-\EE _P\varphi_k(X)\Big]^2\Big\}}{\inf\limits_{P\in\mathcal{P}(\Delta_n,0)}\Bigg\{\sqrt{n\sum\limits_{k\geq 1}\lambda_k[\EE _{P}\varphi_k(X)]^2}-\sqrt{q_{w,1-\alpha}}\Bigg\}^2}\\
	=&0,
	\end{align*}
provided that
\begin{align}
	\inf\limits_{P\in\mathcal{P}(\Delta_n,0)}n\sum\limits_{k\geq 1}\lambda_k[\EE _{P}\varphi_k(X)]^2\rightarrow \infty, \qquad {\rm as\ }n\to\infty.\label{conv1.1}
	\end{align}
	It now suffices to show that (\ref{conv1.1}) holds if $n\Delta_n^2\rightarrow \infty$ as $n\rightarrow \infty$.
	
	To this end, let $u={{d}P}/{{d}P_0}-1$ and
	$$a_k=\langle u,\varphi_k\rangle_{L_2(P_0)}=\EE _P\varphi_k(X)-\EE _{P_0}\varphi_k(X)=\EE _P(\varphi_k(X)).$$
	It is clear the that
	$$\sum\limits_{k\geq 1}\lambda_k^{-1}a_k^2=\|u\|_K^2,\qquad {\rm and}\qquad \sum\limits_{k\geq 1}a_k^2=\|u\|_{L_2(P_0)}^2=\chi^2(P,P_0).$$ 
	By the definition of $\mathcal{P}(\Delta_n,0)$,
	\begin{align*}
	\sup\limits_{P\in\mathcal{P}(\Delta_n,0)}\sum\limits_{k\geq 1}\lambda_k^{-1}a_k^2\leq M^2, \qquad {\rm and} \qquad \inf\limits_{P\in\mathcal{P}(\Delta_n,0)}\sum\limits_{k\geq 1}a_k^2\geq \Delta_n.
	\end{align*}
	Since $n\Delta_n^2\rightarrow \infty$ as $n\rightarrow \infty$, we get
	\begin{align*}
	\inf\limits_{P\in\mathcal{P}(\Delta_n,0)}n\sum\limits_{k\geq 1}\lambda_k[\EE _{P}\varphi_k(X)]^2=&\inf\limits_{P\in\mathcal{P}(\Delta_n,0)}n\sum\limits_{k\geq 1}\lambda_ka_k^2\\
	\geq& \inf\limits_{P\in\mathcal{P}(\Delta_n,0)}n\frac{\Big(\sum\limits_{k\geq 1}a_k^2\Big)^2}{\sum\limits_{k\geq 1}\lambda_k^{-1}a_k^2}\\
	\geq& \frac{n\Delta_n^2}{M^2}\rightarrow \infty
	\end{align*}
	as $n\rightarrow \infty$.
	
\paragraph{Part (ii).} In proving the second part, we will make use of the following lemma that can be obtained by adapting the argument in \citet{gregory1977large}. It gives the limit distribution of V-statistic under $P_n$ such that $P_n$ converges to $P_0$ in the order $n^{-1/2}$.

	\begin{lemma}\label{lemG}
		Consider a sequence of probability measures $\{P_n: n\geq 1\}$ contiguous to $P_0$ satisfying $u_n={{d}P_n}/{{d}P_0}-1\rightarrow 0$ in $L^2(P_0)$. Suppose that for any fixed $k$,
		$$\lim\limits_{n\rightarrow \infty}\sqrt{n}\langle u_n,\varphi_k\rangle_{L^2(P_0)}=\tilde{a}_k, \qquad {\rm and}\qquad\lim\limits_{n\rightarrow \infty}\sum\limits_{k\geq 1}\lambda_k(\sqrt{n}\langle u_n,\varphi_k\rangle_{L^2(P_0)})^2=\sum\limits_{k\geq 1}\lambda_k\tilde{a}_k^2+\tilde{a}_0<\infty,$$
		for some sequence $\{\tilde{a}_k: k\ge 0\}$, then
		\begin{align*}
		\frac{1}{n}\sum\limits_{k\geq 1}\lambda_k\Big[\sum\limits_{i=1}^n\varphi_k(X_i)\Big]^2\stackrel{d}{\rightarrow}\sum_{k\geq 1}\lambda_k(Z_k+\tilde{a}_k)^2+\tilde{a}_0,
		\end{align*}
where $X_1,\ldots, X_n\stackrel{i.i.d}{\sim}P_n$, and $Z_k$s are independent standard normal random variables.
	\end{lemma}

	Write $L(k)=\lambda_k k^{2s}$. By assumption $(\ref{summable})$,
	$$0<\underline{L}:=\inf\limits_{k\geq 1}L(k)\leq\sup\limits_{k\geq 1}L(k):=\overline{L}<\infty.$$ 
	Consider a sequence of $\{P_n: n\geq 1\}$ such that
	$${{d}P_n}/{{d}P_0}-1={C_1\sqrt{\lambda_{k_n}}}[L(k_n)]^{-1}\varphi_{k_n},$$ where $C_1$ is a positive constant and $k_n=\lfloor{C_2n^\frac{1}{4s}}\rfloor$ for some positive constant $C_2$. Both $C_1$ and $C_2$ will be determined later. Since $\sup\limits_{k\geq 1}\|\varphi_k\|_{\infty}<\infty$ and $\lim\limits_{k\rightarrow \infty}\lambda_k=0$, there exists $N_0>0$ such that $P_n$'s are well-defined probability measures for any $n\geq N_0$.
	
	Note that \begin{align*}
	\|u_n\|_K^2=\frac{C_1^2}{L^2(k_n)}\leq \underline{L}^{-2}C_1^2
	\end{align*} and 
	\begin{align*}
	\|u_n\|_{L^2(P_0)}^2=\frac{C_1^2\lambda_{k_n}}{L^2(k_n)}=\frac{C_1^2}{L(k_n)}k_n^{-2s}\geq \overline{L}^{-1}{C_1^2}k_n^{-2s}\sim \overline{L}^{-1}{C_1^2}C_2^{-2s}n^{-1/2},
	\end{align*}
	where $A_n\sim B_n$ means that $\lim\limits_{n\rightarrow \infty}{A_n}/{B_n}=1$.
	Thus, by choosing $C_1$ sufficiently small and $c_0=\frac{1}{2}\overline{L}^{-1}{C_1^2}C_2^{-2s}$, we ensure that $P_n\in \mathcal{P}(c_0n^{-1/2},0)$ for sufficiently large $n$.
	

	To apply Lemma $\ref{lemG}$, we note that
	$$\lim\limits_{n\rightarrow \infty}\|u_n\|_{L^2(P_0)}^2=\lim\limits_{n\rightarrow \infty}\frac{C_1^2\lambda_{k_n}}{L^2(k_n)}=0.$$ 
	In addition, for any fixed $k$,
	$$\tilde{a}_{n,k}=\sqrt{n}\langle u_n,\varphi_k\rangle_{L^2(P_0)}=0$$
	 for sufficiently large $n$, and 
	 $$\sum\limits_{k\geq 1}\lambda_k\tilde{a}_{n,k}^2=\frac{nC_1^2\lambda_{k_n}^2}{L^2(k_n)}=nC_1^2k_n^{-4s}\rightarrow C_1^2C_2^{-4s}$$ 
	 as $n\rightarrow \infty$. Thus, Lemma \ref{lemG} implies that
	\begin{align*}
	n\gamma(\hat{P}_n,P_0)\stackrel{d}{\rightarrow}\sum_{k\geq 1}\lambda_kZ_k^2+C_1^2C_2^{-4s}.
	\end{align*}
	
	Now take $C_2=\left({2C_1^2}/{q_{w,1-\alpha}}\right) ^{{1}/{4s}}$ so that $C_1^2C_2^{-4s}=\frac{1}{2}q_{w,1-\alpha}$. Then
	\begin{align*}
	\varliminf\limits_{n\rightarrow \infty}\beta(T_{\text{MMD}};c_0n^{-1/2},0)\geq& \lim_{n\rightarrow \infty}P_n(n\gamma(\hat{P}_n,P_0)< q_{w,1-\alpha})\\
	=&P\Big(\sum_{k\geq 1}\lambda_kZ_k^2< \frac{1}{2}q_{w,1-\alpha}\Big)>0,
	\end{align*}
	which concludes the proof.
\end{proof}
\vskip 25pt

\begin{proof}[Proof of Theorem \ref{asympm3d}]
	Let $\tilde{K}_n(\cdot,\cdot):=\tilde{K}_{\varrho_n}(\cdot,\cdot)$. Note that
	\begin{align*}
	n\hat{\eta}_{\varrho_n}^2(P,P_0)=&\frac{1}{n}\sum\limits_{k\geq 1}\frac{\lambda_k}{\lambda_k+\varrho_n^2}\Big[\sum\limits_{i=1}^n\varphi_k(X_i)\Big]^2\\=&\frac{1}{n}\sum\limits_{i,j=1}^n\sum\limits_{k\geq 1}\frac{\lambda_k}{\lambda_k+\varrho_n^2}\varphi_k(X_i)\varphi_k(X_j)\\=&\frac{1}{n}\sum\limits_{i,j=1}^n\tilde{K}_n(X_i,X_j).
	\end{align*}
	Thus
	\begin{align*}
	v_n^{-1/2}[n\hat{\eta}_{\varrho_n}^2(P,P_0)-A_n]=2(n^2v_n)^{-1/2}\sum\limits_{j=2}^n\sum\limits_{i=1}^{j-1}\tilde{K}_n(X_i,X_j).
	\end{align*}
	Let $\zeta_{nj}=\sum\limits_{i=1}^{j-1}\tilde{K}_n(X_i,X_j)$. Consider a filtration $\{\mathcal{F}_j: j\geq 1\}$ where $\mathcal{F}_j=\sigma\{X_j:1\leq i\leq j\}$. Due to the assumption that $K$ is degenerate, we have $\EE \varphi_k(X)=0$ for any $k\geq 1$, which implies that
	\begin{align*}
	\EE (\zeta_{nj}|\mathcal{F}_{j-1})=\sum\limits_{i=1}^{j-1}\EE [\tilde{K}_n(X_i,X_j)|\mathcal{F}_{j-1}]=\sum\limits_{i=1}^{j-1}\EE [\tilde{K}_n(X_i,X_j)|X_i]=0,
	\end{align*}
	for any $j\geq 2$. 
	
	Write 
	\begin{align*}
	U_{nm}=\begin{cases}
	0&m=1 \\ 
	\sum\limits_{j=2}^m\zeta_{nj}& m\geq 2
	\end{cases}.
	\end{align*}
	Then for any fixed $n$, $\{U_{nm}\}_{m\geq 1}$ is a martingale with respect to $\{\mathcal{F}_m: m\geq 1\}$ and 
	\begin{align*}
	v_n^{-1/2}[n\hat{\eta}_{\varrho_n}^2(P,P_0)-A_n]=2(n^2v_n)^{-1/2}U_{nn}.
	\end{align*}
	
	We now apply martingale central limit theorem to $U_{nn}$. Following the argument from \citet{hall1984central}, it can be shown that
	\begin{align}
	\Big[\frac{1}{2}n^2\EE \tilde{K}_n^2(X,X')\Big]^{-1/2}U_{nn}\stackrel{d}{\rightarrow}N(0,1),\label{con2}
	\end{align}
	provided that
	\begin{align}
	&[\EE G_n^2(X,X')+n^{-1}\EE \tilde{K}_n^2(X,X')\tilde{K}_n^2(X,X^{''})+n^{-2}\EE \tilde{K}_n^4(X,X')]/[\EE {\tilde{K}_n^2(X,X')}]^2\rightarrow 0,\label{con3}
	\end{align}
	as $n\rightarrow \infty$, where $G_n(x,x')=\EE \tilde{K}_n(X,x)\tilde{K}_n(X,x')$.
	Since
	$$\EE \tilde{K}_n^2(X,X')=\sum\limits_{k\geq 1}\Big(\frac{\lambda_k}{\lambda_k+\varrho_n^2}\Big)^2=v_n,$$
	(\ref{con2}) implies that
	\begin{align*}
	v_n^{-1/2}[n\hat{\eta}_{\varrho_n}^2(P,P_0)-A_n]=\sqrt{2}\cdot\Big(\frac{1}{2}n^2\EE \tilde{K}_n^2(X,X')\Big)^{-1/2}U_{nn}\stackrel{d}{\rightarrow}N(0,2).
	\end{align*}
	It therefore suffices to verify (\ref{con3}).
		
	Note that
	\begin{align*}
	\EE \tilde{K}_n^2(X,X')=\sum\limits_{k\geq 1}\Big(\frac{\lambda_k}{\lambda_k+\varrho_n^2}\Big)^2
	\geq &\sum\limits_{\lambda_k\geq \varrho_n^2}\frac{1}{4}+\frac{1}{4\varrho_n^4}\sum\limits_{\lambda_k<\varrho_n^2}\lambda_k^2\\
	=&\frac{1}{4}|\{k:\lambda_k\geq \varrho_n^2\}|+\frac{1}{4\varrho_n^4}\sum\limits_{\lambda_k<\varrho_n^2}\lambda_k^2\asymp \varrho_n^{-1/s},
	\end{align*}
	where the last step holds by considering that $\lambda_k\asymp k^{-2s}$. 	Hereafter, we shall write $a_n\asymp b_n$ if $0<\varliminf\limits_{n\rightarrow \infty}a_n/b_n\leq \varlimsup\limits_{n\rightarrow \infty}a_n/b_n <\infty$, for two positive sequences $\{a_n\}$ and $\{b_n\}$.
Similarly, 
	\begin{align*}
	\EE G_n^2(X,X')=\sum\limits_{k\geq 1}\Big(\frac{\lambda_k}{\lambda_k+\varrho_n^2}\Big)^4\leq |\{k:\lambda_k\geq\varrho_n^2\}|+\varrho_n^{-8}\sum\limits_{\lambda_k<\varrho_n^2}\lambda_k^4\asymp\varrho_n^{-1/s},
	\end{align*}
	and
	\begin{align*}
	\EE \tilde{K}_n^2(X,X')\tilde{K}_n^2(X,X'')=&\EE \Big\{\sum\limits_{k\geq 1}\Big(\frac{\lambda_k}{\lambda_k+\varrho_n^2}\Big)^2\varphi_k^2(X)\Big\}^2\\\leq &\left(\sup_{k\geq 1}\|\varphi_k\|_{\infty}\right)^4\Big\{\sum\limits_{k\geq 1}\Big(\frac{\lambda_k}{\lambda_k+\varrho_n^2}\Big)^2\Big\}^2\asymp \varrho_n^{-2/s}.
	\end{align*}
	Thus there exists a positive constant $C_3$ such that
	\begin{align}
	\EE G_n^2(X,X')/[\EE \tilde{K}_n^2(X,X')]^2\leq C_3\varrho_n^{1/s}\rightarrow 0,\label{conv1}
	\end{align}
	and
	\begin{align}
	n^{-1}\EE \tilde{K}_n^2(X,X')\tilde{K}_n^2(X,X'')/[\EE {\tilde{K}_n^2(X,X')}]^2\leq C_3n^{-1}\rightarrow \infty,\label{conv2}
	\end{align}
	as $n\rightarrow \infty$. On the other hand,
	\begin{align*}
	\EE \tilde{K}_n^4(X,X')\leq \|\tilde{K}_n\|^2_{\infty}\EE \tilde{K}_n^2(X,X'),
	\end{align*}
	where
	\begin{align*}
	\|\tilde{K}_n\|_{\infty}=\sup\limits_{x}\Bigg\{\sum\limits_{k\geq 1}\frac{\lambda_k}{\lambda_k+\varrho_n^2}\varphi_k^2(x)\Bigg\}\leq\Big(\sup\limits_{k\geq 1}\|\varphi_k\|_{\infty}\Big)^2\sum\limits_{k\geq 1}\frac{\lambda_k}{\lambda_k+\varrho_n^2}\asymp \varrho_n^{-1/s}.
	\end{align*}
	This implies that for some positive constant $C_4$,
	\begin{align}
	n^{-2}\EE \tilde{K}_n^4(X,X')\}/[\EE {\tilde{K}_n^2(X,X')}]^2\leq n^{-2}\|\tilde{K}_n\|^2_{\infty}/\EE {\tilde{K}_n^2(X,X')}\leq C_4(n^2\varrho_n^{1/s})^{-1}\rightarrow 0.\label{conv3}
	\end{align}
	as $n\rightarrow \infty$. Together, (\ref{conv1}), (\ref{conv2}) and (\ref{conv3}) ensure that condition (\ref{con3}) holds.
\end{proof}
\vskip 25pt

\begin{proof}[Proof of Theorem \ref{crm3d}]
	Note that
	\begin{align*}
	&n\eta_{\varrho_n}^2(\hat{P}_n,P_0)-\frac{1}{n}\sum_{i=1}^n\tilde{K}_n(X_i,X_j)\\=&\frac{1}{n}\sum_{k\geq 1}\frac{\lambda_k}{\lambda_k+\varrho_n^2}\sum_{\substack{1\leq i,j\leq n\\ i\neq j}}\varphi_k(X_i)\varphi_k(X_j)\\
	=&\frac{1}{n}\sum_{k\geq 1}\frac{\lambda_k}{\lambda_k+\varrho_n^2}\sum_{\substack{1\leq i,j\leq n\\ i\neq j}}[\varphi_k(X_i)-\EE _P\varphi_k(X)][\varphi_k(X_j)-\EE _P\varphi_k(X)]\\
	&+\frac{2(n-1)}{n}\sum_{k\geq 1}\frac{\lambda_k}{\lambda_k+\varrho_n^2}[\EE _P\varphi_k(X)]\sum_{1\leq i\leq n}[\varphi_k(X_i)-\EE _P\varphi_k(X)]\\
	&+\frac{n(n-1)}{n}\sum_{k\geq 1}\frac{\lambda_k}{\lambda_k+\varrho_n^2}[\EE _P\varphi_k(X)]^2\\
	:=& V_1+V_2+V_3.
	\end{align*}
Obviously, $\EE _PV_1V_2=0$. We first argue that the following three statements together implies the desired result:
\begin{eqnarray}
\label{as2}\lim\limits_{n\rightarrow \infty}\inf\limits_{P\in\mathcal{P}(\Delta_n,\theta)}v_n^{-1/2}V_3&=&\infty,\\
\label{as1a}\sup\limits_{P\in\mathcal{P}(\Delta_n,\theta)}(\EE _PV_1^2/V_3^2)&=&o(1),\\
\label{as1b}\sup\limits_{P\in\mathcal{P}(\Delta_n,\theta)}(\EE _PV_2^2/V_3^2)&=&o(1).
\end{eqnarray}

To see this, note that \eqref{as2} implies that
	\begin{align*}
	&\lim\limits_{n\rightarrow \infty}\inf\limits_{P\in\mathcal{P}(\Delta_n,\theta)}P(v_n^{-1/2}[n\hat{\eta}_{\varrho_n}^2(P,P_0)-A_n]\geq \sqrt{2}z_{1-\alpha})\\
	\geq&\lim_{n\rightarrow \infty}\inf_{P\in\mathcal{P}(\Delta_n,\theta)}P\Big(v_n^{-1/2}V_3\geq 2\sqrt{2}z_{1-\alpha}, V_1+V_2+V_3\geq\frac{1}{2}V_3\Big)\\
	=&\lim_{n\rightarrow \infty}\inf_{P\in\mathcal{P}(\Delta_n,\theta)}P\Big(V_1+V_2+V_3\geq\frac{1}{2}V_3\Big).
	\end{align*}
On the other hand, \eqref{as1a} and \eqref{as1b} imply that
	\begin{align*}
\lim_{n\rightarrow \infty}\inf_{P\in\mathcal{P}(\Delta_n,\theta)}P\Big(V_1+V_2+V_3\geq\frac{1}{2}V_3\Big)
	=&1-\lim_{n\rightarrow \infty}\sup_{P\in\mathcal{P}(\Delta_n,\theta)}P\Big(V_1+V_2+V_3<\frac{1}{2}V_3\Big)\\
	\geq&1-\lim_{n\rightarrow \infty}\sup_{P\in\mathcal{P}(\Delta_n,\theta)}\frac{\EE _P(V_1+V_2)^2}{(V_3/2)^2}=1.
	\end{align*}
This immediately suggests that $T_{\text{M}^3\text{d}}$ is consistent. We now show that \eqref{as2}-\eqref{as1b} indeed hold.

\paragraph{Verifying \eqref{as2}.}
We begin with \eqref{as2}. Since $v_n\asymp \varrho_n^{-1/s}$ and $V_3=(n-1)\eta_{\varrho_n}^2(P,P_0)$, (\ref{as2}) is equivalent to
	\begin{align*}
	\lim\limits_{n\rightarrow \infty}\inf\limits_{P\in\mathcal{P}(\Delta_n,\theta)}n\varrho_n^{\frac{1}{2s}}\eta_{\varrho_n}^2(P,P_0)=\infty. 
	\end{align*}
	
	For any $P\in\mathcal{P}(\Delta_n,\theta)$, let $u={{d}P}/{{d}P_0}-1$ and $a_k=\langle u,\varphi_k\rangle_{L_2(P_0)}=\EE _P\varphi_k(X)$. Based on the assumption that $K$ is universal, $u=\sum\limits_{k\geq 1}a_k\varphi_k$. We consider the case $\theta=0$ and $\theta>0$ separately.
	
\begin{enumerate}	
\item[(1)] First consider $\theta=0$. It is clear that
	\begin{align*}
	\eta_{\varrho_n}^2(P,P_0)=&\sum_{k\geq 1}a_k^2-\sum_{k\geq 1}\frac{\varrho_n^2}{\lambda_k+\varrho_n^2}a_k^2\\
	\geq&\|u\|_{L_2(P_0)}^2-\varrho_n^2\sum_{k\geq 1}\frac{1}{\lambda_k}a_k^2\\
	\geq&\|u\|_{L_2(P_0)}^2-\varrho_n^2M^2.
	\end{align*}
	Take $\varrho_n\le \sqrt{{\Delta_n}/(2M^2)}$ so that $\rho_n^2M^2\le\frac{1}{2}\Delta_n$. Then we have
	\begin{align*}
	\inf\limits_{P\in\mathcal{P}(\Delta_n,0)}\eta_{\varrho_n}^2(P,P_0)\geq\frac{1}{2}\inf\limits_{P\in\mathcal{P}(\Delta_n,0)}\|u\|_{L_2(P_0)}^2=\frac{1}{2}\Delta_n.
	\end{align*}
	
\item[(2)] Now consider the case when $\theta>0$. For $P\in\mathcal{P}(\Delta_n,\theta)$, $\forall$ $R>0$, $\exists$ $f_R\in\mathcal{H}(K)$ such that $\|u-f_R\|_{L_2(P_0)}\leq MR^{-1/\theta}$ and $\|f_R\|_{K}\leq R$. Let $b_k=\langle f_R,\varphi_k\rangle _{L_2(P_0)}$.
\begin{align*}
\eta_{\varrho_n}^2(P,P_0)=&\sum_{k\geq 1}a_k^2-\sum_{k\geq 1}\frac{\varrho_n^2}{\lambda_k+\varrho_n^2}a_k^2\\
\geq&\|u\|_{L_2(P_0)}^2-2\sum_{k\geq 1}\frac{\varrho_n^2}{\lambda_k+\varrho_n^2}(a_k-b_k)^2-2\sum_{k\geq 1}\frac{\varrho_n^2}{\lambda_k+\varrho_n^2}b_k^2\\
\geq&\|u\|_{L_2(P_0)}^2-2\sum_{k\geq 1}(a_k-b_k)^2-2\varrho_n^2\sum_{k\geq 1}\frac{1}{\lambda_k}b_k^2\\
=&\|u\|_{L_2(P_0)}^2-2\|u-f_R\|_{L_2(P_0)}^2-2\varrho_n^2\|f_R\|_{K}^2.
\end{align*}
	Taking $R=({2M}/{\|u\|_{L_2(P_0)}})^{\theta}$ yields that
	\begin{align*}
	\eta_{\varrho_n}^2(P,P_0)\geq \|u\|_{L_2(P_0)}^2-2M^2R^{-2/\theta}-2\varrho_n^2R^2=\frac{1}{2}\|u\|_{L_2(P_0)}^2-2\varrho_n^2R^2.
	\end{align*}
	Now by choosing
	$$
	\varrho_n	\le\frac{1}{2\sqrt{2}}(2M)^{-\theta}\Delta_n^{\frac{1+\theta}{2}},
	$$
	we can ensure that
	$$2\varrho_n^2R^2\leq \frac{1}{4}\|u\|_{L_2(P_0)}^2.$$ So that
%
%
	\begin{align*}
	\inf\limits_{P\in\mathcal{P}(\Delta_n,\theta)}\eta_{\varrho_n}^2(P,P_0)\geq\inf\limits_{P\in\mathcal{P}(\Delta_n,\theta)}\frac{1}{4}\|u\|_{L_2(P_0)}^2\geq\frac{1}{4}\Delta_n.
	\end{align*} 
	
\end{enumerate}

	In both cases, with $\varrho_n\leq C\Delta_n^{\frac{\theta+1}{2}}$ for a sufficiently small $C=C(M)>0$,  $\lim\limits_{n\rightarrow\infty}\varrho_n^{\frac{1}{2s}}n\Delta_n=\infty$ suffices to ensure (\ref{as2}) holds. Under the condition that $\lim\limits_{n\rightarrow \infty}\Delta_nn^{\frac{4s}{4s+\theta+1}}=\infty$, 
$$
\varrho_n=cn^{-\frac{2s(\theta+1)}{4s+\theta+1}}\leq C\Delta_n^{\frac{\theta+1}{2}}
$$
for sufficiently large $n$ and $\lim\limits_{n\rightarrow\infty}\varrho_n^{\frac{1}{2s}}n\Delta_n=\infty$ holds as well.

	
\paragraph{Verifying \eqref{as1a}.}
Rewrite $V_1$ as
	\begin{align*}
	V_1=&\frac{1}{n}\sum_{\substack{1\leq i,j\leq n\\ i\neq j}}\sum_{k\geq 1}\frac{\lambda_k}{\lambda_k+\varrho_n^2}[\varphi_k(X_i)-\EE _P\varphi_k(X)][\varphi_k(X_j)-\EE _P\varphi_k(X)]\\
	:=& \frac{1}{n}\sum_{\substack{1\leq i,j\leq n\\ i\neq j}}F_n(X_i,X_j).
	\end{align*}
	Then
	\begin{align*}
	\EE _PV_1^2=&\frac{1}{n^2}\sum_{\substack{i\neq j\\ i'\neq j'}}\EE _PF_n(X_i,X_j)F_n(X_{i'},X_{j'})\\
	=&\frac{2n(n-1)}{n^2}\EE _PF_n^2(X,X')\\
	\leq& 2\EE _PF_n^2(X,X').
	\end{align*}
	
	Recall that, for any two random variables $Y_1$, $Y_2$ such that $\EE Y_1^2<\infty$,
	\begin{align*}
	\EE [Y_1-\EE (Y_1|Y_2)]^2=\EE Y_1^2-\EE [\EE (Y_1|Y_2)^2]\leq \EE Y_1^2.
	\end{align*}
	Therefore,
	\begin{align*}
	\EE _PF_n^2(X,X')\leq \EE _P\{\tilde{K}_n(X,X')-\EE _P[\tilde{K}_n(X,X')|X]\}^2\leq \EE _{P}\tilde{K}_n^2(X,X').
	\end{align*}
	Thus, to prove \eqref{as1a}, it suffices to show that
	\begin{align*}
	\lim\limits_{n\rightarrow \infty}\sup\limits_{P\in\mathcal{P}(\Delta_n,\theta)}\EE _{P}\tilde{K}_n^2(X,X')/V_3^2=0.
	\end{align*}

	For any $g\in L_2(P_0)$ and positive definite kernel $G(\cdot,\cdot)$ such that $\EE _{P_0}G^2(X,X')<\infty$, let
	\begin{align*}
	\|g\|_G:= \sqrt{\EE _{P_0}[g(X)g(X')G(X,X')]}.
	\end{align*}
	By the positive definiteness of $G(\cdot,\cdot)$, triangular inequality holds for $\|\cdot\|_G$, \ie, for any $g_1$, $g_2\in L_2(P_0)$, 
	\begin{align*}
	|\|g_1\|_G-\|g_2\|_G|\leq \|g_1-g_2\|_G,
	\end{align*}
	which implies that
	\begin{align}
	\Bigg|\sqrt{\EE _P\tilde{K}_n^2(X,X')}-\sqrt{\EE _{P_0}\tilde{K}_n^2(X,X')}\Bigg|\leq \sqrt{\EE _{P_0}[u(X)u(X')\tilde{K}_n^2(X,X')]}.\label{tri}
	\end{align}
	We now appeal to the following lemma to bound the right hand side of \eqref{tri}:
	\begin{lemma}\label{schur2}
	Let $G$ be a Mercer kernel defined over $\Xcal\times \Xcal$ with eigenvalue-eigenfunction pairs $\{(\mu_k, \varphi_k): k\ge 1\}$ with respect to $L_2(P)$ such that $\mu_1\ge \mu_2\ge\cdots$. If $G$ is a trace kernel in that $\EE G(X,X)<\infty$, then for any $g\in L_2(P)$
	$$
	\EE _{P}[g(X)g(X')G^2(X,X')]\le \mu_1\left(\sum_{k\ge 1}\mu_k\right)\left(\sup_{k\ge 1}\|\varphi_k\|_\infty\right)^2\|g\|_{L_2(P)}^2.
	$$
	\end{lemma}
	By Lemma \ref{schur2}, we get
	\begin{align*}
	\EE _{P_0}[u(X)u(X')\tilde{K}_n^2(X,X')]\leq &C_5\left(\sum\limits_k\frac{\lambda_k}{\lambda_k+\varrho_n^2}\right)\|u\|_{L_2(P_0)}^2\asymp \varrho_n^{-1/s}\|u\|_{L_2(P_0)}^2.
	\end{align*}
	Recall that
	$$\EE _{P_0}\tilde{K}^2_n(X,X')=\sum\limits_k\left(\frac{\lambda_k}{\lambda_k+\varrho_n^2}\right)^2\asymp \varrho_n^{-1/s}.$$ In the light of (\ref{tri}), they imply that
	\begin{align*}
	\EE _{P}\tilde{K}^2_n(X,X')\leq 2\{\EE _{P_0}\tilde{K}^2_n(X,X')+\EE _{P_0}[u(X)u(X')\tilde{K}_n^2(X,X')]\}\leq C_6\varrho_n^{-1/s}[1+\|u\|_{L_2(P_0)}^2].
	\end{align*}
On the other hand, it is not hard to verify that with our choice of $\varrho_n$,
	\begin{align*}
	\frac{1}{4}\|u\|_{L_2(P_0)}^2\leq \eta_{\varrho_n}^2(P,P_0)\leq \|u\|_{L_2(P_0)}^2,
	\end{align*}
	for any $P\in\mathcal{P}(\Delta_n,\theta)$. Thus
	\begin{align*}
	&\lim\limits_{n\rightarrow \infty}\sup\limits_{P\in\mathcal{P}(\Delta_n,\theta)}\EE _P\tilde{K}_n^2(X,X')/V_3^2\\\leq &16C_6\Big\{\Big(\lim\limits_{n\rightarrow \infty}\inf\limits_{P\in\mathcal{P}(\Delta_n,\theta)}\varrho_n^{1/s}n^2\|u\|_{L_2(P_0)}^4\Big)^{-1}+\Big(\lim\limits_{n\rightarrow \infty}\inf\limits_{P\in\mathcal{P}(\Delta_n,\theta)}\varrho_n^{1/s}n^2\|u\|_{L_2(P_0)}^2\Big)^{-1}\Big\}=0
	\end{align*}
	provided that $\lim\limits_{n\rightarrow \infty}n^{\frac{4s}{4s+\theta+1}}\Delta_n=\infty$. This immediately implies \eqref{as1a}.
	
\paragraph{Verifying \eqref{as1b}.} Observe that
	\begin{align*}
	\EE _PV_2^2\leq &4n\EE _P\Big\{\sum\limits_{k\geq 1}\frac{\lambda_k}{\lambda_k+\varrho_n^2}[\EE _P\varphi_k(X)][\varphi_k(X)-\EE _P\varphi_k(X)]\Big\}^2\\
	\leq &4n\EE _P\Big\{\sum\limits_{k\geq 1}\frac{\lambda_k}{\lambda_k+\varrho_n^2}[\EE _P\varphi_k(X)][\varphi_k(X)]\Big\}^2\\
          =& 4n\EE _{P_0}\left([1+u(X)]\Big\{\sum\limits_{k\geq 1}\frac{\lambda_k}{\lambda_k+\varrho_n^2}[\EE _P\varphi_k(X)][\varphi_k(X)]\Big\}^2\right).
	\end{align*}
	
It is clear that
	\begin{align*}
	&\EE _{P_0}\Big\{\sum\limits_{k\geq 1}\frac{\lambda_k}{\lambda_k+\varrho_n^2}[\EE _P\varphi_k(X)][\varphi_k(X)]\Big\}^2\\=&\sum\limits_{k,k'\geq 1}\frac{\lambda_k}{\lambda_k+\varrho_n^2}\frac{\lambda_{k'}}{\lambda_{k'}+\varrho_n^2}\EE _P\varphi_k(X)\EE _P\varphi_{k'}(X)\EE _{P_0}[\varphi_k(X)\varphi_{k'}(X)]\\=&\sum\limits_{k\geq 1}\Big(\frac{\lambda_k}{\lambda_k+\varrho_n^2}\Big)^2[\EE _P\varphi_k(X)]^2\leq \eta_{\varrho_n}^2(P,P_0).
	\end{align*}
On the other hand,
\begin{align*}
	&\EE _{P_0}\left(u(X)\Big\{\sum\limits_{k\geq 1}\frac{\lambda_k}{\lambda_k+\varrho_n^2}[\EE _P\varphi_k(X)][\varphi_k(X)]\Big\}^2\right)\\
	\leq&\sqrt{\EE _{P_0}\left(u^2(X)\Big\{\sum\limits_{k\geq 1}\frac{\lambda_k}{\lambda_k+\varrho_n^2}[\EE _P\varphi_k(X)][\varphi_k(X)]\Big\}^2\right)}\times\\
	&\times\sqrt{\EE _{P_0}\Big\{\sum\limits_{k\geq 1}\frac{\lambda_k}{\lambda_k+\varrho_n^2}[\EE _P\varphi_k(X)][\varphi_k(X)]\Big\}^2}\\
	\leq&\|u\|_{L_2(P_0)}\sup\limits_{x}\Big|\sum\limits_{k\geq 1}\frac{\lambda_k}{\lambda_k+\varrho_n^2}[\EE _P\varphi_k(X)][\varphi_k(x)]\Big|\cdot\eta_{\varrho_n}(P,P_0)\\
	\leq&\left(\sup\limits_{k}\|\varphi_k\|_{\infty}\right)\|u\|_{L_2(P_0)}\sum\limits_{k\geq 1}\frac{\lambda_k}{\lambda_k+\varrho_n^2}|\EE _P\varphi_k(X)|\cdot\eta_{\varrho_n}(P,P_0)\\
	\leq&\left(\sup\limits_{k}\|\varphi_k\|_{\infty}\right)\|u\|_{L_2(P_0)}\sqrt{\sum\limits_{k\geq 1}\frac{\lambda_k}{\lambda_k+\varrho_n^2}}\sqrt{\sum\limits_{k\geq 1}\frac{\lambda_k}{\lambda_k+\varrho_n^2}[\EE _P\varphi_k(X)]^2}\cdot\eta_{\varrho_n}(P,P_0)\\\leq &C_7\|u\|_{L_2(P_0)}\varrho_n^{-\frac{1}{2s}}\eta_{\varrho_n}^2(P,P_0).
	\end{align*}
Together, they imply that
\begin{align*}
	&\lim\limits_{n\rightarrow \infty}\sup\limits_{P\in\mathcal{P}(\Delta_n,\theta)}\EE _PV_1^2/V_3^2\\\leq &4\max\{1,C_7\}\Bigg\{\Big(\lim\limits_{n\rightarrow \infty}\inf\limits_{P\in\mathcal{P}(\Delta_n,\theta)}n\eta_{\varrho_n}^2(P,P_0)\Big)^{-1}+\lim\limits_{n\rightarrow \infty}\sup\limits_{P\in\mathcal{P}(\Delta_n,\theta)}\Bigg(\frac{\|u\|_{L_2(P_0)}}{\varrho_n^{\frac{1}{2s}}n\eta_{\varrho_n}^2(P,P_0)}\Bigg)\Bigg\}=0,
	\end{align*}
under the assumption that $\lim\limits_{n\rightarrow \infty}n^{\frac{4s}{4s+\theta+1}}\Delta_n=\infty$.
\end{proof}
\vskip 25pt

\begin{proof}[Proof of Theorem \ref{cr}]
Without loss of generality, assume $M=1$ and $\Delta_n=cn^{-\frac{4s}{4s+\theta+1}}$ for some $c>0$. The main idea behind our proof is to carefully construct a finite subset of $\mathcal{P}(\Delta_n,\theta)\setminus \{P_0\}$, and show that one can not reliably distinguish $P_0$ from an unknown instance from this subset based on a sample of $n$ observations. We shall consider the cases of $\theta=0$ and $\theta>0$ separately.

\paragraph{The case of $\theta=0$.} 
We first treat the case when $\theta=0$. Let $K_n=\lfloor{C_8\Delta_n^{-\frac{1}{2s}}}\rfloor$ for a sufficiently small constant $C_8>0$ and $a_n=\sqrt{{\Delta_n}/{K_n}}$. For any $\xi_n:=(\xi_{n1},\xi_{n2},\cdots,\xi_{nK_n})^\top\in \{\pm 1\}^{K_n}$, write
$$
u_{\xi_n} = a_n\sum_{k=1}^{K_n}\xi_{nk}\varphi_k.
$$
It is clear that
$$\|u_{n,\xi_n}\|_{L_2(P_0)}^2=K_na_n^2=\Delta_n$$
and	
$$
\|u_{n,\xi_n}\|_\infty\le a_nK_n\left(\sup\limits_{k}\|\varphi_k\|_{\infty}\right)\asymp \Delta_n^{2s-1\over 4s}\rightarrow 0.
$$
By taking $C_8$ small enough, we can also ensure
$$\|u_{\xi_n}\|_K^2=a_n^2\sum\limits_{k=1}^{K_n}\lambda_k^{-1}\leq 1,$$

Therefore, there exists a probability measure $P_{\xi_n}\in\mathcal{P}(\Delta_n, 0)$ such that $dP/dP_0=1+u_{n,\xi_n}$. Following a standard argument for minimax lower bound, it suffices to show that
\begin{equation}
\label{eq:min}
\varlimsup\limits_{n\rightarrow \infty}\EE_{P_0} \left({1\over 2^{K_n}}\sum_{\xi_n\in \{\pm 1\}^{K_n}}\left\{\prod\limits_{i=1}^{n}[1+u_{\xi_n}(X_i)]\right\}\right)^2<\infty.
\end{equation}
See, \eg, \cite{ingster2003nptest,tsybakov2008introduction}.

Note that
	\begin{align*}
&\EE_{P_0} \left({1\over 2^{K_n}}\sum_{\xi_n\in \{\pm 1\}^{K_n}}\left\{\prod\limits_{i=1}^{n}[1+u_{\xi_n}(X_i)]\right\}\right)^2\\
=&\EE_{P_0} \left({1\over 2^{2K_n}}\sum_{\xi_n,\xi_n'\in \{\pm 1\}^{K_n}}\left\{\prod\limits_{i=1}^{n}[1+u_{\xi_n}(X_i)]\right\}\left\{\prod\limits_{i=1}^{n}[1+u_{\xi_n'}(X_i)]\right\}\right)\\
	=&{1\over 2^{2K_n}}\sum_{\xi_n,\xi_n'\in \{\pm 1\}^{K_n}}\prod\limits_{i=1}^n\EE_{P_0} \Bigg
	\{[1+u_{\xi_n}(X_i)][1+u_{\xi_n'}(X_i)]\Bigg\}\\
	=&{1\over 2^{2K_n}}\sum_{\xi_n,\xi_n'\in \{\pm 1\}^{K_n}}\Big(1+a_{n}^2\sum\limits_{k=1}^{K_n}\xi_{nk}\xi_{nk}'\Big)^n\\
	\leq&{1\over 2^{2K_n}}\sum_{\xi_n,\xi_n'\in \{\pm 1\}^{K_n}}\exp\Big(na_n^2\sum\limits_{k=1}^{K_n}\xi_{n,k}\xi_{n,k}'\Big)\\
	=&\left\{\frac{\exp(na_{n}^2)+\exp(-na_{n}^2)}{2}\right\}^{K_n}\\
	\leq&\exp\Big
	\{K_n\Big(\frac{\exp(na_{n}^2)+\exp(-na_{n}^2)}{2}-1\Big)\Big\}.
	\end{align*}
An application of Taylor expansion shows that there exist $t_0>0$ and $C_9>0$ such that
	\begin{align*}
	{t^2}-C_9|t|^3\leq \exp(t)+\exp(-t)\leq{t^2}+C_9|t|^3
	\end{align*}
	for any $|t|\leq t_0$. With the particular choice of $K_n$, $a_n$, and the conditions on $\Delta_n$, this immediately implies \eqref{eq:min}.
	
\paragraph{The case of $\theta>0$.} The main idea is similar to before. To find a set of probability measures in $\mathcal{P}(\Delta_n, \theta)$, we appeal to the following lemma.

\begin{lemma}\label{char0}
Let $u=\sum\limits_{k}a_k\varphi_k$. If
\begin{align*}
\left(\sum\limits_{k=1}^K\frac{a_k^2}{\lambda_k}\right)^{2/\theta}\left(\sum\limits_{k\geq K}a_k^2\right)\leq M^2,
\end{align*}
then $u\in \mathcal{F}(\theta, M)$.
\end{lemma}

Similar to before, we shall now take $K_n=\lfloor{C_{10}\Delta_n^{-\frac{\theta+1}{2s}}}\rfloor$ and $a_{n}=\sqrt{{\Delta_n}/{K_n}}$. By Lemma \ref{char0}, we can find $P_{\xi_n}\in\mathcal{P}(\Delta_n, \theta)$ such that $dP/dP_0=1+u_{n,\xi_n}$, for appropriately chosen $C_{10}$. Following the same argument as in the previous case, we can again verify \eqref{eq:min}.
\end{proof}
\vskip 25pt

\begin{proof}[Proof of Theorem \ref{crtm3d}]
	Without loss of generality, assume that $\Delta_n(\theta)= c_1({n}^{-1}{\sqrt{\log\log n}})^{\frac{4s}{4s+\theta+1}}$ for some constant $c_1>0$ to be determined later. 

\noindent{\bf Type I Error.} We first prove the first statement which shows that the Type I error converges to $0$. Following the same notations as defined in Theorem \ref{asympm3d}, let
	\begin{align*}
	N_{n,2}=\EE \Big\{\sum\limits_{j=2}^n\EE \Big(\tilde{\zeta}_{nj}^2|\mathcal{F}_{j-1}\Big)-1\Big\}^2,\quad L_{n,2}=\sum\limits_{j=2}^n\EE \tilde{\zeta}_{nj}^4
	\end{align*}
where $\tilde{\zeta}_{nj}={\sqrt{2}\zeta_{nj}}/({n\sqrt{v_n}})$. As shown by \citet{haeusler1988rate},
	\begin{align*}
	\sup\limits_{t}|P(T_{n,\varrho_n}>t)-\bar{\Phi}(t)|\leq C_{11}(L_{n,2}+N_{n,2})^{1/5},
	\end{align*}
	where $\bar{\Phi}(t)$ is the survival function of the standard normal, \ie, $\bar{\Phi}(t)=P(Z>t)$ where $Z\sim N(0,1)$. Again by the argument from \citet{hall1984central},
	\begin{align*}
	\EE \Big\{\sum\limits_{j=2}^n\EE (\zeta_{nj}^2|\mathcal{F}_{j-1})-\frac{1}{2}n(n-1)v_n\Big\}^2\leq C_{12}[n^4\EE G_n^2(X,X')+n^3\EE \tilde{K}_n^2(X,X')\tilde{K}_n^2(X,X^{''})],
	\end{align*}
where $G_n(\cdot,\cdot)$ is defined in the proof of Theorem \ref{asympm3d}, and 
	\begin{align*}
	\sum\limits_{j=2}^n\EE \zeta_{nj}^4\leq C_{13}[n^2\EE \tilde{K}_n^4(X,X')+n^3\EE \tilde{K}_n^2(X,X')\tilde{K}_n^2(X,X^{''})],
	\end{align*}
	which ensure
	\begin{align*}
	N_{n,2}=&\frac{4\EE \Big\{\sum\limits_{j=2}^n\EE (\zeta_{nj}^2|\mathcal{F}_{j-1})-\frac{1}{2}n(n-1)v_n-\frac{1}{2}nv_n\Big\}^2}{n^4v_n^2}\\\leq &8\max\left\{C_{12},\frac{1}{4}\right\}\Bigg\{\frac{\EE G_n^2(X,X')}{v_n^2}+\frac{\EE \tilde{K}_n^2(X,X')\tilde{K}_n^2(X,X^{''})}{nv_n^2}+\frac{1}{n^2}\Bigg\},
	\end{align*}
	and
	\begin{align*}
	L_{n,2}=\frac{4\sum\limits_{j=2}^n\EE \tilde{\zeta}_{nj}^4}{n^4v_n^2}\leq 4C_{13}\Bigg\{\frac{\EE \tilde{K}_n^4(X,X')}{n^2v_n^2}+\frac{\EE \tilde{K}_n^2(X,X')\tilde{K}_n^2(X,X^{''})}{nv_n^2}\Bigg\}.
	\end{align*}
	
	As shown in the proof of Theorem \ref{asympm3d},
	\begin{align*}
	\frac{\EE G_n^2(X,X')}{v_n^2}\leq C_3\varrho_n^{1/s},\quad \frac{\EE \tilde{K}_n^4(X,X')}{n^2v_n^2}\leq C_4n^{-2}\varrho_n^{-1/s},\quad {\rm and}\quad \frac{\EE \tilde{K}_n^2(X,X')\tilde{K}_n^2(X,X^{''})}{nv_n^2}\leq C_3n^{-1}.
	\end{align*}
	Therefore,
	\begin{align*}
	\sup\limits_{t}|P(T_{n,\varrho_n}>t)-\bar{\Phi}(t)|\leq C_{14}(\varrho_n^{\frac{1}{5s}}+n^{-\frac{1}{5}}+n^{-\frac{2}{5}}\varrho_n^{-\frac{1}{5s}}),
	\end{align*}
which implies that
	\begin{align*}
	P\left(\sup\limits_{\varrho_n\in\Lambda_n}T_{n,\varrho_n}>t\right)\leq m_*\bar{\Phi}(t)+C_{15}(2^{m_*\over5s}\varrho_*^{1\over 5s}+m_*n^{-\frac{1}{5}}+n^{-\frac{2}{5}}\varrho_*^{-\frac{1}{5s}}),\qquad \forall t.
	\end{align*}
	It is not hard to see, by the definitions of $m_*$,
$$2^{m_*}\varrho_*\leq 2\left(\frac{\sqrt{\log\log n}}{n}\right)^{\frac{2s}{4s+1}}$$
and
	\begin{align*}
	m_*=&(\log 2)^{-1}\{2s\log n-\frac{2s}{4s+1}\log n+o(\log n)\}\\=&(\log 2)^{-1}\frac{8s^2}{4s+1}\log n+o(\log n)\asymp \log n.
	\end{align*}
Together with the fact that $\bar{\Phi}(t)\leq {1\over 2} e^{-{t^2}/{2}}$ for $t\geq 0$, we get
	\begin{align*}
	&P\left(\sup\limits_{0\leq k\leq m_*}T_{n,2^k\varrho_*}>\sqrt{3\log\log n}\right)\\\leq &C_{16}\left[e^{-\frac{3}{2}\log\log n}\log n+\left(\frac{\sqrt{\log n}}{n}\right)^{\frac{2}{5(4s+1)}}+n^{-\frac{1}{5}}\log\log n+n^{-\frac{2}{5}}\left(\frac{\sqrt{\log \log n}}{n}\right)^{-\frac{2}{5}}\right]\rightarrow 0,
	\end{align*}
	as $n\rightarrow \infty$.
	
\paragraph{Type II Error.} Next consider Type II error. To this end, write $\varrho_n(\theta)=({\sqrt{\log\log n}\over n})^{2s(\theta+1)\over 4s+\theta+1}$. Let
$$\tilde{\varrho}_ n(\theta)=\sup\limits_{0\leq k\leq m_*}\{2^{k}\varrho_\ast: \varrho_n\leq\varrho_n(\theta)\}.$$
It is clear that $\tilde{T}_n\ge T_{n,\tilde{\varrho}_n(\theta)}$ for any $\theta\ge 0$. It therefore suffices to show that for any $\theta\ge 0$,
$$
\lim_{n\to\infty}\inf\limits_{\theta\geq 0}\inf_{P\in\mathcal{P}(\Delta_n,\theta)}P\left\{T_{n,\tilde{\varrho}_n(\theta)}\geq \sqrt{3\log\log n}\right\}=1.
$$
By Markov inequality, this can accomplished by verifying
	\begin{align}
	\inf\limits_{\theta\in[0,\infty)}\inf\limits_{P\in\mathcal{P}(\Delta_n(\theta),\theta)}\EE _PT_{n,\tilde{\varrho}_n(\theta)}\geq \tilde{M}\sqrt{\log\log n}\label{mean}
	\end{align}
	for some $\tilde{M}>\sqrt{3}$; and
	\begin{align}
	\lim\limits_{n\rightarrow\infty}\sup\limits_{\theta\ge 0}\sup\limits_{P\in\mathcal{P}(\Delta_n(\theta),\theta)}\frac{\mathrm{Var}\left(T_{n,\tilde{\varrho}_n(\theta)}\right)}{\left(\EE _PT_{n,\tilde{\varrho}_n(\theta)}\right)^2}=0\label{mv}.
	\end{align}

We now show that both (\ref{mean}) and (\ref{mv}) hold with $$\Delta_n(\theta)=c_1\left({\sqrt{\log\log n}\over n}\right)^{\frac{4s}{4s+\theta+1}}$$ for a sufficiently large $c_1=c_1(M,\tilde{M})$.

Note that $\forall$ $\theta\in[0,\infty)$, 
	\begin{align}
	\frac{1}{2}\varrho_n(\theta)\leq\tilde{\varrho_n}(\theta)\leq \varrho_n(\theta),\label{lub}
	\end{align}，
	which immediately suggests 
	\begin{align}
	\eta_{\tilde{\varrho}_n(\theta)}^2(P,P_0)\geq\eta_{\varrho_n(\theta)}^2(P,P_0).\label{ineq}
	\end{align}
	
Following the arguments in the proof of Theorem \ref{crm3d},
$$\EE _PT_{n,\tilde{\varrho}_n(\theta)}\geq C_{17}n[\tilde{\varrho}_n(\theta)]^{1/(2s)}\eta_{\tilde{\varrho}_n(\theta)}^2(P,P_0)\geq 2^{-1/(2s)}C_{17}n[\varrho_n(\theta)]^{1/2s}\eta_{\varrho_n(\theta)}^2(P,P_0),$$ and $\forall\ P\in\Pcal(\Delta_n(\theta),\theta)$, 
\begin{align}
\eta_{\varrho_n(\theta)}^2(P,P_0)\geq {1\over 4}\|u\|_{L_2(P_0)}^2\label{ineq2}
\end{align}
provided that $\Delta_n(\theta)\geq C'(M)\left({\sqrt{\log\log n}\over n}\right)^{\frac{4s}{4s+\theta+1}}$.

Therefore,
$$\inf\limits_{P\in\Pcal(\Delta_n(\theta),\theta)}\EE_PT_{n,\tilde{\varrho}_n(\theta)}\geq C_{18}n[\varrho_n(\theta)]^{1/(2s)}\Delta_n(\theta)\geq C_{18}c_1\sqrt{\log\log n}\geq \tilde{M}\sqrt{\log\log n}$$
if $c_1\geq C_{18}^{-1}\tilde{M}$. Hence to ensure (\ref{mean}) holds, it suffices to take
$$c_1=\max\{C'(M),C_{18}^{-1}\tilde{M}\}.$$

	With (\ref{lub}), (\ref{ineq}) and (\ref{ineq2}), the results in Theorem 3 imply that for sufficiently large $n$
	\begin{align*}
	\sup\limits_{P\in\mathcal{P}(\Delta_n^*(\theta),\theta,M)}\frac{\mathrm{Var}\left(T_{n,\tilde{\varrho}_n(\theta)}\right)}{\left(\EE _PT_{n,\tilde{\varrho}_n(\theta)}\right)^2}\leq& C_{19}\Big\{\left([\varrho_n(\theta)]^{\frac{1}{2s}}n\Delta_n^*(\theta)\right)^{-2}+\left([\varrho_n(\theta)]^{\frac{1}{s}}n^2\Delta_n^*(\theta)\right)^{-1}\\&+(n\Delta_n^*(\theta))^{-1}+\left([\varrho_n(\theta)]^{\frac{1}{2s}}n\sqrt{\Delta_n^*(\theta)}\right)^{-1}\Big\}\\
	\leq&2C_{19}\left([\varrho_n(\theta)]^{\frac{1}{2s}}n\Delta_n^*(\theta)\right)^{-1}= 2C_{19}(c_1\log\log n)^{-\frac{1}{2}}\rightarrow \infty,
	\end{align*}
	which shows (\ref{mv}).
\end{proof}
\vskip 25pt

\begin{proof}[Proof of Theorem \ref{cra}]
The main idea of the proof is similar to that for Theorem \ref{cr}. To this end, assume, without loss of generality, that
$$\Delta_n(\theta)=c_2\left(\frac{n}{\sqrt{\log\log n}}\right)^{-\frac{4s}{4s+\theta+1}},\qquad \forall\theta\in[\theta_1,\theta_2],$$ 
where $c_2>0$ is a sufficiently small constant to be determined later.

	
	Let $r_n=\lfloor C_{20}\log n\rfloor$ and $K_{n,1}=\lfloor{C_{21}\Delta_n^{-\frac{\theta_1+1}{2s}}(\theta_1)}\rfloor$for sufficiently small $C_{20},C_{21}>0$. Set $\theta_{n,1}=\theta_1$.  For $2\leq r\leq r_n$, let
	$$
	K_{n,r}=2^{r-2}K_{n,1}
	$$
%
%
and $\theta_{n,r}$ is selected such that the following equation holds.
	\begin{align*}
	K_{n,r}=\left\lfloor C_{21}[\Delta_n(\theta_{n,r})]^{-\frac{\theta_{n,r}+1}{2s}}\right\rfloor.
	\end{align*}
	Note that by choosing $C_{20}$ sufficiently small,
	\begin{align*}
	K_{n,r_n}=2^{r_n-2}K_{n,1}\leq &\left\lfloor c_2^{\frac{2(\theta_1+1)}{4s+\theta+1}}\left(\frac{n}{\sqrt{\log\log n}}\right)^{\frac{2(\theta_1+1)}{4s+\theta+1}}\cdot 2^{r_n-2}\right\rfloor\\=&\left\lfloor c_2^{\frac{2(\theta_1+1)}{4s+\theta+1}}\exp\left(\log\left(\frac{n}{\sqrt{\log\log n}}\right)\cdot\frac{2(\theta_1+1)}{4s+\theta_1+1}+(r_n-2)\log 2\right)\right\rfloor\\\leq &\left\lfloor C_{21}\exp\left(\log\left(\frac{n}{\sqrt{\log\log n}}\right)\cdot\frac{2(\theta_2+1)}{4s+\theta_2+1}\right)\right\rfloor=\lfloor C_{21}[\Delta_n(\theta_{2})]^{-\frac{\theta_{2}+1}{2s}}\rfloor
	\end{align*}
	for sufficiently large $n$. Thus, we can guarante that $\forall$ $1\leq r\leq r_n$, $\theta_{n,r_n}\in[\theta_1,\theta_2]$.

We now construct a finite subset of $\cup_{\theta\in [\theta_1,\theta_2]}\mathcal{P}(\Delta_n(\theta),\theta)$ as follows. For each $\xi_{n,r}=(\xi_{n,r,1},\cdots,\xi_{n,r,K_{n,r}})\in\{\pm 1\}^{K_{n,r}}$, let
\begin{align*}
f_{n,r,\xi_{n,r}}=1+\sum\limits_{k=K^*_{n,r-1}+1}^{K^*_{n,r}}a_{n,r}\xi_{n,r,k}\varphi_k,
\end{align*}
where $K^*_{n,r}=K_{n,1}+\cdots+K_{n,r}$, and $a_{n,r}=\sqrt{\Delta_n(\theta_{n,r})/K_{n,r}}$. Following the same argument as that in the proof of Theorem \ref{cr}, we can verify that with a sufficiently small $C_{21}$, each $P_{n,r,\xi_{n,r}}\in \mathcal{P}(\Delta_n(\theta_{n,r}),\theta_{n,r})$, where $f_{n,r,\xi_{n,r}}$ is the Radon-Nikodym derivative $dP_{n,r,\xi_{n,r}}/dP_0$. With slight abuse of notation, write
\begin{align*}
f_n(X_1,X_2,\cdots,X_n)={1\over r_n}{\sum\limits_{r=1}^{r_n}f_{n,r}(X_1,X_2,\cdots,X_n)},
\end{align*}
where
\begin{align*}
f_{n,r}(X_1,X_2,\cdots,X_n)={1\over 2^{K_{n,r}}}\sum_{\xi_{n,r}\in \{\pm 1\}^{K_{n,r}}}\prod\limits_{i=1}^n f_{n,r,\xi_{n,r}}(X_i).
\end{align*}
It now suffices to show that
$$
\|f_n\|_{L_2(P_0)}:=\EE_{P_0}f_n^2(X_1,X_2,\cdots,X_n)\to 1,\qquad {\rm as\ }n\to\infty.
$$
%
	
	Note that
	\begin{align*}
\|f_n\|^2_{L_2(P_0)}=&\frac{1}{r_n^2}\sum\limits_{1\leq r,r'\leq r_n}\langle f_{n,r},f_{n,r'}\rangle _{L_2(P_0)}\\
	=&\frac{1}{r_n^2}\sum\limits_{1\leq r\leq r_n}\|f_{n,r}\|^2_{L_2(P_0)}+\frac{1}{r_n^2}\sum\limits_{\substack{1\leq r,r'\leq r_n\\r\neq r'}}\langle f_{n,r},f_{n,r'}\rangle_{L_2(P_0)}.
	\end{align*}
It is easy to verify that, for any $r\neq r'$,
	\begin{align*}
	\langle f_{n,r},f_{n,r'}\rangle_{L_2(P_0)}=1.
	\end{align*}
It therefore suffices to show that
\begin{align*}
\sum\limits_{1\leq r\leq r_n}\|f_{n,r}\|^2_{L_2(P_0)}=o(r_n).
	\end{align*} 
	
Following the same derivation as that in the proof of Theorem \ref{cr}, we can show that
	\begin{align*}
	\|f_{n,r}\|_{L_2(P_0)}^2\leq&\left(\frac{\exp(na_{n,r}^2)+\exp(-na_{n,r}^2)}{2}\right)^{K_{n,r}}\leq\exp(K_{n,r}n^2a_{n,r}^4)
	\end{align*}
	for sufficiently large $n$. By setting $c_2$ in the expression of $\Delta_n(\theta)$ sufficiently small, we have
	\begin{align*}
	K_{n,r}n^2a_{n,r}^4\leq \frac{1}{2}\log r_n,
	\end{align*} 
    which ensures that
	\begin{align*}
	\|f_{n,r}\|_{L_2(P_0)}^2\leq r_n^{1/2}=o(r_n).
	\end{align*}
\end{proof}

\bibliographystyle{myplainnat}
\bibliography{mmdref}

\newpage
\begin{appendices}
%
%

\begin{proof}[Proof of Lemma \ref{schur2}]
	We have
	\begin{align*}
		G^2(x,x')=\sum\limits_{k,l}\mu_k\mu_l\varphi_k(x)\varphi_l(x)\varphi_k(x')\varphi_l(x').
	\end{align*}
Thus
	\begin{align*}
		&\int g(x)g(x')G^2(x,x'){d}P(x){d}P(x')\\
		=&\sum\limits_{k,l}\mu_k\mu_l\Big(\int g(x)\varphi_k(x)\varphi_l(x){d}P(x)\Big)^2\\
		\leq& \mu_1\sum\limits_{k}\mu_k\sum\limits_{l}\Big(\int g(x)\varphi_k(x)\varphi_l(x){d}P(x)\Big)^2\\
		\leq &\mu_1\left(\sum\limits_{k}\mu_k\int g^2(x)\varphi^2_k(x){d}P(x)\right)\\
		\leq &\mu_1\left(\sum\limits_{k}\mu_k\right)\left(\sup\limits_{k}\|\varphi_k\|_\infty\right)^2\|g\|_{L_2(P)}^2.
	\end{align*}
\end{proof}

%
\vskip 25pt

\begin{proof}[Proof of Lemma \ref{char0}]
For brevity, write
$$l_K=\sum_{k=1}^K{a_k^2\over \lambda_k}.$$
By definition, it suffices to show that $\forall$ $R>0$, $\exists$ $f_R\in\mathcal{H}(K)$ such that $\|f_R\|_K^2\leq R^2$ and $\|u-f_R\|_{L_2(P_0)}^2\leq M^2R^{-2/\theta}$.

To this end, let $K$ be such that $l_K^2\le R^2\le l_{K+1}^2$, and denote by
	$$f_R=\sum\limits_{k=1}^Ka_k\varphi_k+a_{K+1}^*(R)\varphi_{K+1},$$ 
	where \begin{align*}
		a_{K+1}^*(R)=\sgn(a_{K+1})\sqrt{\lambda_{K+1}(R^2-l_K^2)}.
	\end{align*}
Clearly,
	\begin{align*}
		\|f_R\|_K^2=\sum\limits_{k=1}^K\frac{a_k^2}{\lambda_k}+\frac{(a_{k+1}^*(R))^2}{\lambda_{K+1}}=R^2,
	\end{align*}
	and
	\begin{align*}
		\|u-f_R\|_{L_2(P_0)}^2=\sum\limits_{k>K+1}a_k^2+\left(|a_{K+1}|-\sqrt{\lambda_{K+1}(R^2-l_K^2)}\right)^2\leq \sum\limits_{k\geq K+1}a_k^2.
	\end{align*}
	
To ensure $u\in\mathcal{F}(\theta,M)$, it suffices to have
	\begin{align*}
		\sup\limits_{l_K^2\leq R^2\leq l_{K+1}^2}\|u-f_R\|_{L_2(P_0)}^2R^{2/\theta}\leq M^2,\quad\forall\  K\geq 0,
	\end{align*}
which concludes the proof.
\end{proof}
\end{appendices}

\end{document}